\documentclass{article}

\usepackage{subcaption}
\usepackage{fullpage}

\usepackage{amsmath, amssymb, amsthm, mathtools}
\usepackage[english]{babel}

\usepackage{enumitem}
\usepackage{lmodern, float}
\usepackage{sidecap}

\usepackage[utf8]{inputenc} 
\usepackage[T1]{fontenc}    
\usepackage{hyperref}       
\usepackage{cleveref}
\usepackage{url}            
\usepackage{booktabs}       
\usepackage{amsfonts}       
\usepackage{nicefrac}       
\usepackage{microtype}      
\usepackage{braket}

\DeclareMathOperator{\sgn}{sgn}

\usepackage[colorinlistoftodos]{todonotes}

\usepackage{wrapfig}

\newtheorem{theorem}{Theorem}[section]
\newtheorem*{theorem*}{Theorem}

\newtheorem{proposition}[theorem]{Proposition}
\newtheorem*{proposition*}{Proposition}
\newtheorem{lemma}[theorem]{Lemma}
\newtheorem*{lemma*}{Lemma}
\newtheorem{corollary}[theorem]{Corollary}
\newtheorem*{conjecture*}{Conjecture}
\newtheorem{fact}[theorem]{Fact}
\newtheorem*{fact*}{Fact}

\newtheorem*{hypothesis*}{Hypothesis}

\newtheorem{claim}[theorem]{Claim}
\newtheorem*{claim*}{Claim}

\theoremstyle{definition}
\newtheorem{definition}[theorem]{Definition}

\theoremstyle{remark}
\newtheorem{remark}[theorem]{Remark}
\newtheorem*{remark*}{Remark}

\newtheorem*{observation*}{Observation}

\newcommand{\R}{\mathbb{R}}
\newcommand{\N}{\mathbb{N}}

\newcommand{\poly}{\mathrm{poly}}

\newcommand{\Bigabs}[1]{\Big\lvert#1\Big\rvert}

\newcommand{\norm}[1]{\lVert #1 \rVert}

\newcommand{\iprod}[1]{\langle#1\rangle}

\newcommand{\Esymb}{\mathbb{E}}
\newcommand{\Psymb}{\mathbb{P}}

\DeclareMathOperator*{\E}{\Esymb}

\DeclareMathOperator*{\ProbOp}{\Psymb}

\renewcommand{\Pr}{\ProbOp}

\renewcommand{\epsilon}{\varepsilon}

\newcommand{\sdpval}{\text{SDP}_{val}}
\newcommand{\sdpvalsq}{\text{SDP}^{(C \log n)}_{val}}

\newcommand{\YES}{\textsc{Yes}~}
\newcommand{\NO}{\textsc{No}~}
\newcommand{\QP}{$\mathcal{QP}$}

\newif\ifnotes\notesfalse

\ifnotes
\usepackage{color}
\definecolor{mygrey}{gray}{0.50}
\newcommand{\notename}[2]{{\textcolor{mygrey}{\footnotesize{\bf (#1:} {#2}{\bf ) }}}}

\else

\newcommand{\notename}[2]{{}}

\fi

\newcommand{\anote}[1]{{\notename{Aravindan}{#1}}}
\newcommand{\dnote}[1]{{\notename{Abhro}{#1}}}

\DeclarePairedDelimiter\inner{\langle}{\rangle}
\DeclarePairedDelimiter\abs{\lvert}{\rvert}

\title{On Robustness to Adversarial Examples and Polynomial Optimization}

\author{%
Pranjal Awasthi\\
  Department of Computer Science\\
  Rutgers University\\
  \texttt{pranjal.awasthi@rutgers.edu} \\
 \and
Abhratanu Dutta\\
  Department of Computer Science\\
  Northwestern University\\
  \texttt{abhratanudutta2020@u.northwestern.edu} \\
\and
Aravindan Vijayaraghavan\\
  Department of Computer Science\\
  Northwestern University\\
  \texttt{aravindv@northwestern.edu} \\
}
\date{}

\begin{document}

\maketitle
\begin{abstract}

We study the design of computationally efficient algorithms with provable guarantees, that are robust to adversarial (test time) perturbations. While there has been an proliferation of recent work on this topic due to its connections to test time robustness of deep networks, there is limited theoretical understanding of several basic questions like \emph{(i) when and how can one design provably robust learning algorithms?} \emph{(ii) what is the price of achieving robustness to adversarial examples in a computationally efficient manner?}

The main contribution of this work is to exhibit a strong connection between achieving robustness to adversarial examples, and a rich class of polynomial optimization problems, thereby making progress on the above questions. In particular, we leverage this connection to (a) design \anote{removed: the {\em first}} computationally efficient robust algorithms with provable guarantees for a large class of hypothesis, namely linear classifiers and degree-$2$ polynomial threshold functions~(PTFs), (b) give a precise characterization of the price of achieving robustness in a computationally efficient manner for these classes, (c) design efficient algorithms to certify robustness and generate adversarial attacks in a principled manner for $2$-layer neural networks. We empirically demonstrate the effectiveness of these attacks on real data.

\end{abstract}
\section{Introduction}
\label{sec:intro}
The empirical success of deep learning has led to numerous unexplained phenomena about which our current theoretical understanding is limited. Examples include the ability of complex models to generalize well and effectiveness of first order methods on optimizing training loss. The focus of 
this paper is on the phenomenon of \emph{adversarial robustness}, that was first pointed out by Szegedy et al.~\cite{szegedy2013intriguing}. On many benchmark data sets, deep networks optimized on the training set can often be fooled into misclassifying a test example by making a
small adversarial perturbation that is imperceptible to a human labeler. 
This has led to a proliferation of work on designing robust algorithms that defend against such adversarial perturbations, as well as attacks that aim to break these defenses. 

In this work we choose to focus on \emph{perturbation defense}, the most widely studied formulation of adversarial robustness~\cite{madry2017towards}. 
In the perturbation defense model, given a classifier $f$, an adversary can take a test example $x$ generated from the data distribution and 
perturb it to $\tilde{x}$ such that $\|x-\tilde{x}\| \leq \delta$. Here $\delta$ characterizes the amount of power the adversary has and the distance is typically measured in the $\ell_\infty$ norm (other norms that have been studied include the $\ell_2$ norm).  Given a loss function $\ell(\cdot)$, the goal is to optimize the {\em robust loss} defined as
\anote{Put in an equation.}
$$\mathbb{E}_{(x,y) \sim D} \Big[\max_{\tilde{x}: \|x-\tilde{x}\|_\infty \leq \delta} \ell(f(\tilde{x}), y) \Big].$$
One would expect that when $\delta$ is small the label $y$ of an example does not change, thereby motivating the robust loss objective. 
Despite a recent surge in efforts to theoretically understand adversarial robustness~\cite{xu2009robustness, xu2012robustness, yin2018rademacher, khim2018adversarial, schmidt2018adversarially, feige2015learning, attias2018improved, cullina2018pac,gilmer2018adversarial, tsipras2018robustness, mahloujifar2018curse, mahloujifar2018can, diochnos2018adversarial}, several central questions remain open. 
 \emph{How can one design provable polynomial time algorithms that are robust to adversarial perturbations? Given a classifier and a test input, how can one provably construct an adversarial example in polynomial time or certify that none exists? What computational barriers exist when designing adversarially robust learning algorithms?}


In this work we identify and study a natural class of \emph{polynomial optimization} problems that are intimately connected to adversarial robustness, and help us shed new light on all three of the above questions simultaneously! As a result we obtain the first polynomial time learning algorithms for a large class of functions that are optimally robust to adversarial perturbations. Furthermore, we also provide nearly matching computational intractability results that, together with our upper bounds give a sharp characterization of the price of achieving adversarial robustness in a computationally efficient manner. We now summarize our main results.
\anote{de-emphasized first.}

\noindent \textbf{Our Contributions}
\anote{Should we change title to Polynomial Optimization and Finding Adversarial Examples.}
\noindent{\bf Polynomial optimization and Adversarial Robustness.} 
We identify a natural class of polynomial optimization problems that provide a common and principled framework for studying various aspects of adversarial robustness. These problems are also closely related to a rich class of well-studied problems that include the Grothendi\"{e}ck problem and its generalizations~\cite{alonnaor, charikar2004maximizing, alon2006quadratic, khotnaor}. Given a classifier of the form $sgn(g(x))$ with $g:\R^n \to \R$, input $x$, and budget $\delta > 0$, the optimization problem is
$$ \max_{z \in \R^n: \norm{z}_\infty \le \delta} g(x+z).$$
\anote{Made an equation}
Usually, such problems are NP-hard and one relaxes them to find a $\hat{z}$ such that $g(x+\hat{z})$ comes as close to $g(x+z^*)$ in the objective value, where $z^*$ is the optimal solution. We instead require the algorithm to output a $\hat{z}$ such that $g(x+\hat{z}) \geq g(x+z^*)$ at the cost of violating the $\ell_\infty$ constraint by a factor $\gamma \geq 1$. An efficient algorithm for producing such a $\hat{z}$ leads to an adversarial attack (in the relaxed $\ell_\infty$ neighborhood of radius $\gamma \delta$) when an adversarial example exists. On the other hand, if the algorithm produces no $\widehat{z}$, then this guarantees that there is no adversarial example within the $\ell_\infty$ neighborhood of radius $\delta$. 
\anote{Added above line.}
We then design such algorithms based on convex programming relaxations to get the \emph{first} provable polynomial time adversarial attacks when the given classifier is a degree-$1$ or a degree-$2$ polynomial threshold function~(PTF).

\noindent{\bf Algorithms for Learning Adversarially Robust Classifiers.} 
Next we use the algorithm for finding adversarial examples to design polynomial time algorithms for learning robust classifiers for the class of degree-$1$ and degree-$2$ polynomial threshold functions~(PTFs).
\anote{Too much going on in the sentence. Moved things around..}
To incorporate robustness we introduce a parameter $\gamma$, that helps clarify the tradeoff when computational efficiency is desired. 
\anote{Why does it need to be a PTF below?}
We focus on the $0/1$ error and say that a class $\mathcal{F}$ of PTFs of VC dimension $\Delta$ is $\gamma$-approximately robustly learnable if there exists a (randomized) polynomial time algorithm that, for any given $\epsilon, \delta > 0$, takes as input $\mathrm{poly}(\Delta, \frac 1 \epsilon)$ examples generated from a distribution and labeled by a function in $\mathcal{F}$ that has zero $\delta$-robust error~(realizable case), outputs a classifier from $\mathcal{F}$ that has $(\delta / \gamma)$-robust error upper bounded by $\epsilon$. 
See Section~\ref{sec:model} for the formal definition. We design polynomial time algorithms for degree-$1$ and degree-$2$ PTFs with $\gamma=1$ and $\gamma=O(\sqrt{\log n})$ respectively. Our next result that we discuss below a nearly matching lower bound. Together this gives nearly optimal approximately robust polynomial time algorithms for learning PTFs of degree at most $2$. 



\noindent{\bf Computational Hardness.} While our algorithm for degree-$1$ PTFs is optimal, i.e., has $\gamma=1$, for degree-$2$ and higher PTFs, we show that one indeed has to pay a price for computational robustness. We establish this by proving that robust learning of degree-$2$ PTFs is computationally hard for $\gamma = o(\log^c n)$, for some constant $c > 0$~(see Section~\ref{sec:lowerbound} for formal statements). 
This is in sharp contrast to the non-robust setting ($\delta=0$), where there exist polynomial time algorithms for constant degree PTFs (in the literature this is referred to as proper PAC learning in the realizable setting). More importantly, our lower bound again leverages the connection to polynomial optimization and in fact shows that robust learning of degree-$2$ PTFs for $\gamma=o(\sqrt{\eta_{approx}})$ is NP-hard where $\eta_{approx}$ is \emph{precisely} the hardness of approximation factor of a well-studied combinatorial optimization problem called {\em Quadratic Programming}. 
\anote{used to be: the corresponding polynomial optimization problem.} 
Hence, any significant improvement in the approximation factor in our upper bound is unlikely. While our hardness result applies to algorithms that output a classifier of low error, we also prove a more robust hardness result showing that for learning degree-$2$ and higher PTFs without any loss in the robustness parameter, i.e, $\gamma=1$, it is computationally hard to even find a classifier of any constant error in the range $(0,\frac 1 4)$. 


\noindent{\bf Application to Neural Networks.} Finally, we show that the connection to polynomial optimization also leads to new algorithms for generating adversarial attacks on neural networks. We focus on $2$-layer neural networks with ReLU activations. We show that given a network and a test input, the problem of finding an adversarial example can also be phrased as an optimization problem of the kind studied for PTFs. \anote{used to be: corresponds to a natural optimization problem.} We design a semi-definite programming~(SDP) based polynomial time algorithm to generate an adversarial attack for such networks and compare our attack to the state-of-the-art attack of Madry et al.~\cite{madry2017towards} on the MNIST data set. 


\noindent{\bf Paper Outline.}
In the rest of the paper, we give an overview of related work in Section~\ref{sec:related}. We define our model formally and give an overview of our techniques in Section~\ref{sec:model}. We then describe the connection to polynomial optimization in Section~\ref{sec:optimization} and use it to design robust learning algorithms in Section~\ref{sec:upper-bound}, and derive computational intractability results in Section~\ref{sec:lowerbound}. In Section~\ref{sec:nn}, we design adversarial attacks for $2$ layer neural networks, followed by conclusions in Section~\ref{sec:conclusions}.

\section{Related Work}
\label{sec:related}
As mentioned in the introduction, there has been a recent explosion of works on understanding adversarial robustness from both empirical and theoretical aspects. Here we choose to discuss the theoretical works that are the most relevant to our paper. We refer the interested reader to a recent paper by~\cite{gilmer2018motivating} for a broader discussion. Prior to their relevance for deep networks, robust optimization problems have been studied in machine learning and other domains. The works of~\cite{bhattacharyya2004robust, globerson2006nightmare, shivaswamy2006second} studies optimization heuristics for optimizing a robust loss that can handle noisy or missing data. The works of~\cite{xu2009robustness, xu2012robustness} proved an equivalence between robust optimization and various regularized variants of SVMs. They used this relation to re-derive standard generalization bounds for SVMs and their kernel versions. Akin to classifier stability, these bounds depend on the robustness of the classifier on the training set. A recent work of~\cite{bietti2018regularization} views deep networks as functions in an RKHS and designs new norm based regularization algorithms to achieve robustness. 

Motivated by connections to deep networks a recent line of work studies generalization bounds for robust learning. The work of~\cite{schmidt2018adversarially} provides specific constructions of a linear binary classification task where a single example is enough to learn the problem in the usual sense, i.e., to achieve low test error, whereas learning the problem robustly requires a significantly large training set. The authors also show that in certain cases, non-linearity can help reduce the sample complexity of robust learning. The work of~\cite{cullina2018pac} proposes a PAC model for robust learning and defines adversarial VC dimension as a combinatorial quantity that captures robust learning via robust empirical risk minimization~(ERM). The authors show that for linear classifiers the adversarial VC dimension is the same as the VC dimension, although there are functions classes and distributions where the gap between the two quantities could be much higher. The recent works of~\cite{yin2018rademacher} and~\cite{khim2018adversarial} analyze Rademacher complexity of robust loss functions classes. In particular, it is observed that even for linear models with bounded weight norm, there is an unavoidable dependence on the data dimension in the Rademacher complexity of robust loss function classes. These results point to the fact that for many distributions robust learning could require many more training samples than their non-robust counterpart. The work of~\cite{feige2015learning, attias2018improved} studies algorithms and generalization bounds for a model where the adversary can choose perturbations from a known finite set of small size $k$. 

Another recent line of work studies the trade-off between traditional test error and robust error. The work of~\cite{tsipras2018robustness} designs a classification task that is efficiently learnable with a linear classifier to low standard error, but has the property that any classifier that achieves
low test error will have high robust error on the task. The work of~\cite{gilmer2018adversarial} designs a task that is learnable by a degree-2 polynomial and relates the test error of any model to its robust error. Similar conclusions have been observed in~\cite{mahloujifar2018curse, mahloujifar2018can, diochnos2018adversarial} and have been used to design various data poisoning attacks. These results essentially follows from the use of isoperimetric inequalities for distributions such as the Gaussian and the uniform distribution over the Boolean hypercube. However, as noted in~\cite{gilmer2018adversarial}, it is not clear if the same relation holds between test error and robust error for real world data distributions. The work of~\cite{fawzi2016robustness} relates robustness to the curvature of the decision boundary and uses it to quantify robustness to random perturbations.

Yet another line of work concerns the design of certificates of perturbation robustness or distributional robustness of a given classifier (e.g., deep neural networks) at a given point~\cite{wong2018provable, raghunathan2018certified, sinha2018certifying}. This is achieved by the use of convex relaxations of the optimal robustness at a given point. These works also conclude that by augmenting the training objective with a penalty that depends on the certificates, one can empirically achieve increased robustness. However these algorithms do not give any guarantees for relating the bound achieved by the certificate of robustness to the optimal robustness around a given point. 

The work of Bubeck et al.~\cite{bubeck2018adversarial2, bubeck2018adversarial} provides a cryptographic lower bound by designing a computational task in $\mathbb{R}^n$ that is robustly learnable using $\mathrm{poly}(n)$ samples to any given robustness parameter $M$, but is hard to learn robustly to any non-trivial robustness parameter $\epsilon > 0$, in polynomial time. When translated to our model, this provides an instance of a cryptographic learning task that is computationally hard to $\gamma$-approximately robustly learn for any constant $\gamma \geq 1$. However, this does not rule out the possibility that natural function classes can be robustly learned without any loss in robustness parameter. Our result rules this out for the class of degree-$2$ and higher PTFs, even in the realizable setting, i.e., when there exists a robust classifier of zero error! Finally, to the best of our knowledge, our upper bounds are the first to establish the robustness tradeoff for computationally efficient learning for a large natural class of functions.


\section{Model and Preliminaries}
\label{sec:model}
We focus on binary classification, and adversarial perturbations are measured in $\ell_\infty$ norm. For a vector $x \in \mathbb{R}^n$, we have $\|x\|_\infty = \max_i |x_i|$. We study robust learning of \emph{polynomial threshold functions}~(PTFs). These are functions of the form $sgn(p(x))$, where $p(x)$ is a polynomial in $n$ variables over the reals. Here $sgn(t)$ equals $+1$, if $t \geq 0$ and $-1$ otherwise. Given $y,y' \in \{-1,1\}$, we study the $0/1$ loss defined as $\ell(y,y') = 1$ if $y \neq y'$ and $0$ otherwise. Given a binary classifier $\sgn(g(x))$, an input $x^*$, and a budget $\delta > 0$, we say that $x^*+z$ is an {\em adversarial example} (for input $x^*$) if $\sgn(g(x^*+z)) \neq \sgn(g(x^*))$ and that $\|z\|_\infty \leq \delta$. One could similarly define the notion of adversarial examples for other norms. For a classifier with multiple outputs, we say that $x^*+z$ is an adversarial example iff the largest co-ordinate of $g(x^* + z)$ differs from the largest co-ordinate of $g(x^*)$. \anote{argmax over $z$? Also why isn't it $x^*$ that is called an adversarial example?} We now define the notion of robust error of a classifier.
\begin{definition}[$\delta$-robust error]\label{def:robust-error}
Let $f(x)$ be a Boolean function mapping $\mathbb{R}^n$ to $\{-1, 1\}$. Let $D$ be a distribution over $\mathbb{R}^n \times \{-1, 1\}$. Given $\delta > 0$, we define the $\delta$-robust error of $f$ with respect to $D$ as $err_{\delta, D}(f) = \mathbb{E}_{(x,y) \sim D} \big[ \sup_{z \in B^n_\infty(0,\delta)} \ell(f(x+z),y)\big]$.
Here $B^n_\infty(0,\delta)$ denotes the $\ell_\infty$ ball of radius $\delta$, i.e., $B^n_\infty(0,\delta) = \{x \in \mathbb{R}^n: \|x\|_\infty \leq \delta\}$.
\end{definition}
Analogous to empirical error in PAC learning, we denote $\hat{err}_{\delta, S}(f)$ to be the $\delta$-robust empirical error of $f$, i.e., the robust error computed on the given sample $S$.
To bound generalization gap, we will use the notion of adversarial VC dimension as introduced in~\cite{cullina2018pac}. 
Next we define robust learning for PTFs.
\begin{definition}[$\gamma$-approximately robust learning]
Let $\mathcal{F}$ be the class of degree-$d$ PTFs from $\mathbb{R}^n \mapsto \{-1,1\}$ of VC dimension $\Delta = O(n^d)$. For $\gamma \geq 1$, an algorithm $\mathcal{A}$ $\gamma$-approximately robustly learns $\mathcal{F}$ if the following holds for any $\epsilon, \delta, \eta > 0$: Given $m = \mathrm{poly}(\Delta, \frac 1 \epsilon, \frac 1 \eta)$ samples from a distribution $D$ over $\mathbb{R}^n \times \{-1,1\}$, if $\mathcal{F}$ contains a function $f^* $ such that $err_{\delta,D}(f^*)=0$, then with probability at least $1-\eta$, $\mathcal{A}$ runs in time polynomial in $m$ and outputs $f \in \mathcal{F}$ such that
$
err_{\delta/ \gamma, D} (f) \leq \epsilon.
$
If $\mathcal{F}$ admits such an algorithm then we say that $\mathcal{F}$ is $\gamma$-approximately robustly learnable. Here $\gamma$ quantifies the price of achieving computationally efficient robust learning, with $\gamma=1$ implying optimal learnability.
\end{definition}

\paragraph{A Note about the Model and the Realizability Assumption}
Our definition of an adversarial example requires that $\sgn(g(x^* + z)) \neq \sgn(g(x^*))$, whereas for robust learning we require a classifier that satisfies $\sgn(g(x^*+z)) \neq y$, where $y$ is the given label of $x^*$. This might create two sources of confusion to the reader: a) In general the two requirements might be incompatible, and b) It might happen that initially $\sgn(g(x^*))$ predicts the true label incorrectly but there is a perturbation $z$ such that $\sgn(g(x^*+z))$ predicts the true label correctly. In this case one may not count $z$ as an adversarial example. To address (a) we would like to stress that all our guarantees hold under the realizability assumption, i.e., we assume that there is true function $c^*$ such that for all examples $x$ in the support of the distribution and all perturbations of magnitude upto $\delta$, $\sgn(c^*(x^*+z)) = \sgn(c^*(x^*))$. Hence, there will indeed be a target concept for which no adversarial example exists and as a result will have zero robust error. To address (b) we would like to point out that in Section~\ref{sec:upper-bound} where we use the subroutine for finding adversarial examples to learn a good classifier $\sgn(g)$, we always enforce the constraint that on the training set $\sgn(g(x^*)) = \sgn(c^*(x^*))$ and $g$ is as robust as possible. Hence when we find an adversarial example for 
a point $x^*$ in our training set, it will indeed satisfy that $\sgn(g(x^*+z)) \neq \sgn(c^*(x))$ and correctly penalize $g$ for the mistake. More generally, we could also define an adversarial example as one where given pair $(x^*,y)$ the goal is to find a $z$ such that $\sgn(g(x^*+z)) \neq y$. All of our guarantees from Section~\ref{sec:optimization} apply to this definition as well. Finally, in the non-realizable case, the distinction between defining adversarial robustness as either $\sgn(g(x^*+z)) \neq \sgn(g(x^*))$, or $\sgn(g(x^*+z)) \neq y$, or even $\sgn(g(x^*+z)) \neq \sgn(c^*(x))$ matters and has different computational and statistical implications~\cite{diochnos2018adversarial, gourdeau2019hardness}. Understanding when one can achieve computationally efficient robust learning in the non-realizable case is an important direction for future work.

The definition of  $\gamma$-approximately robustly learnability has the realizability assumption built into it. So, when we prove that a class $\mathcal{F}$ is $\gamma$-approximately robustly learnable, we find an approximate robust learner from $\mathcal{F}$ under the realizability assumption on $\mathcal{F}$ i.e. for a set of points from the distribution, the algorithm guarantees to return an approximate robust learner only if there exists a perfect robust learner in the class $\mathcal{F}$ of learners.

The work of~\cite{cullina2018pac} defines the notion of adversarial VC dimension to bound the generalization gap for robust empirical risk minimization. Additionally, the authors show that for linear classifiers the adversarial VC dimension remains the same as that of the original class. The bound below then follows by viewing PTFs as linear classifiers in a higher dimensional space.
\begin{lemma}
\label{lem:uniform-convergence}
Let $\mathcal{F}$ be a class of degree-$d$ polynomial threshold functions from $\mathbb{R}^n \mapsto \{-1, 1\}$ of VC dimesion $\Delta = O(n^d)$. Given $\delta,\eta > 0$, and a set $S$ of $m$ examples $(x_1, y_1), \dots, (x_m, y_m)$ generated from a distribution $D$ over $\mathbb{R}^n \times \{-1, 1\}$, with probability at least $1-\eta$, we have that $\sup_{f \in \mathcal{F}} |err_{\delta, D}(f) - \hat{err}_{\delta, S}(f)| \leq 2\sqrt{2\Delta \log m/ m} + \sqrt{\log(1/\eta)/(2m)}
$.
\end{lemma}


\section{Finding Adversarial Examples using Polynomial Optimization}
\label{sec:optimization}

In this section we introduce the broad class of polynomial optimization problems which are useful in designing adversarial (test-time) examples with provable guarantees for polynomial threshold functions (PTFs), and depth-$2$ neural networks with RELU gates. These polynomial optimization problems are generalizations of well-studied combinatorial optimization problems like the {\em Groth\"endieck problem} and computing operator norms of matrices. We then design algorithms with provable guarantees for some of these classes.
Proposition~\ref{prop:advexamples} restated below illustrates the connection and motivates the family of optimization problems that arise when designing algorithms with provable guarantees for finding adversarial examples for $sgn(g(x))$. While our theory below is stated for binary classifiers, it is easily extended to multiclass classification.

\begin{proposition}\label{prop:advexamples}
Let $\gamma\ge 1$. There is an efficient algorithm that given a classifier $sgn(f(x))$ and a point $x^*$, and budget $\delta > 0$, guarantees to either (a) find an adversarial example in $B_\infty^n(x^*,\gamma\delta)$,  or (b) certify the absence of any adversarial example in $B_\infty^n(x^*, \delta)$, given access to an efficient optimization algorithm that takes $x^*$ and a polynomial $g(z) \in \set{f(x^*+z),-f(x^*+z)}$ as input and finds a $\widehat{z}$ such that
$ g(\widehat{z}) \ge \max_{\norm{z}_\infty \le \delta} g(z)$ with $\norm{\widehat{z}}_\infty \le \gamma \delta$.
\end{proposition}

\begin{proof}[Proof of Proposition~\ref{prop:advexamples}]
Let $\textsc{ALG}_\gamma$ be the optimization algorithm.
Suppose there exists an adversarial example $x^*+z^*$ with $\norm{z^*}_\infty \le \delta$, and let $y^*:=sgn(f(x^*))$ be the label for the point $x^*$. Then we have that $\max_{z : \norm{z}_\infty \le \delta} (-y^*)  f(x^*+z) >(-y^*) f(x^*+z^*) > 0$. Now for $g(z)=-y^*f(x^*+z)$ (a polynomial in $z$),  we get that $\textsc{ALG}_\gamma$ finds a point $\widehat{z}$ with $\norm{\widehat{z}}_\infty \le \gamma \delta$ that also satisfies $(-y^*) f(x^*+\widehat{z}) >0$ i.e.,  $sgn(f(x^*)) \ne sgn(f(x^*+\widehat{z}))$, as required. Furthermore, if $\textsc{ALG}_\gamma$ fails, i.e., outputs a $\hat{z}$ such that $(-y^*) f(x^*+\widehat{z}) < 0$, then from the guarantee of the algorithm we know that $\max_{z : \norm{z}_\infty \le \delta} (-y^*) f(x^*+z) < 0$ and hence no adversarial example exists within a $\delta$ ball around $x^*$.
\end{proof}
The proposition above also holds for randomized algorithms. While the proof of the proposition only requires that the algorithm returns $\widehat{z}$ with $g(\widehat{z})>0$, it effectively requires that $\widehat{z}$ attains at least as large an objective value because the constant term can be arbitrary.
When the classifier is a degree-$d$ PTF of the form $sgn(f)$, it leads to the following approximate optimization problem: given as input a degree $d$ polynomial $g:\R^n \to \R$ (potentially different from $f$) and any $\eta, \delta >0$, find in time $\mathrm{poly}(n, \log(\frac 1 \eta))$ and w.p. at least $1-\eta$ a point $\hat{x}$ s.t.
   \begin{equation}
      g(\hat{x}) \geq \max_{x \in B^n_\infty(0,\delta)} g(x)
  \text{ and }  \widehat{x} \in B^n_{\infty}(0,\gamma \delta).
   \end{equation}

The above problem is closely related to the standard approximation variant of polynomial maximization problem where the goal is to obtain, in polynomial time, an objective value as close to the optimal one, without violating the $\|\|_\infty$ ball constraint. Instead, our problem asks for the same objective value at the cost of an increase in the radius of the optimization ball. \footnote{In approximation algorithms literature this will correspond to obtaining a $(1,\gamma)$-bicriteria approximation.}  This changes the flavor of the problem, and introduces new challenges particularly when the polynomial $g$ is non-homogenous.

We begin with the following simple claim about degree-$1$ PTFs.
\begin{claim}\label{thm:admissible-degree-1}
There is a deterministic linear-time algorithm that given any linear threshold function $sgn(b^Tx +c)$, a point $x^*$ and $\delta>0$, provably finds an adversarial example in the $\ell_\infty$ ball of $\delta$ around $x^*$ when it exists.
\end{claim}
\begin{proof}
We use Proposition~\ref{prop:advexamples} applied with linear functions. For linear function $g(x)$ represented by $g(x):= b^T x + c$ where $b \in \R^n, c \in \R$, we can easily find a solution $\widehat{x} \in B^n_\infty(0,\delta)$ such that
$g(\widehat{x}) = \max_{x \in B^n_\infty(0,\delta)} g(x)$.
This is because the linear form $b^T x + c$ is maximized within $B^n_\infty(0,\delta)$ by setting each variable $x_i$ to be $\delta$ if the corresponding $b_i \geq 0$, and $-\delta$, otherwise.
\end{proof}
As we will see in Section~\ref{sec:upper-bound}, this will further be used to give robust learning algorithms for linear threshold functions.
Our main theoretical result of this section gives an algorithm for provably finding adversarial examples for degree-$2$ PTFs.
\begin{theorem}\label{thm:QP:algorithm}
For any $\delta, \eta>0$, there is a polynomial time algorithm that given a degree-$2$ PTF $sgn(f(x))$ and a example $(x^*,sgn(f(x^*)))$, guarantees at least one of the following holds with probability at least $(1-\eta)$:  (a) finds an adversarial example $(x^*+\widehat{z})$ i.e., $sgn(f(x^*)) \ne sgn(f(x^*+\widehat{z}))$, with $ \norm{\widehat{z}}_\infty \le C\delta \sqrt{\log n}$,  or (b) certifies that $\forall z: \norm{z}_\infty \le \delta$, $sgn(f(x^*))=sgn(f(x^*+z))$ for some constant $C>0$.
\end{theorem}
To establish the above theorem using Proposition~\ref{prop:advexamples}, we need to design a polynomial time algorithm that given any degree-$2$ polynomial $g(x)=x^T A x + b^T x +c$ with $A \in \R^{n \times n}, b \in \R^n, c \in \R$, finds a solution $\widehat{x}$ with $\norm{\widehat{x}}_\infty \le O(\sqrt{\log n})\cdot \delta$ such that
$g(\widehat{x}) \ge \max_{\norm{x}_\infty \le \delta} g(x)$.

\begin{figure}
\begin{center}
\fbox{\parbox{0.98\textwidth}{
\begin{enumerate}[topsep=0pt,itemsep=-1ex,partopsep=1ex,parsep=1ex]
\item Given $(A,b,c)$ that defines the polynomial $g(z):= z^T A z+ b^T z +c$.
\item Solve the SDP given by following vector program:\\
$\max ~ \sum_{i,j} A_{ij} \iprod{u_i, u_j}+ \sum_{i} b_i \iprod{u_i, u_0} + c$ subject to $\norm{u_i}^2_2 \le \delta^2 ~\forall i \in [n]$,  $\norm{u_0}_2^2=1$.
\item Let $u_i^{\perp}$ represent the component of $u_i$ orthogonal to $u_0$. Draw $\zeta \sim N(0, I)$ a standard Gaussian vector, and set $\widehat{z}_i  := \iprod{u_i,u_0}+\iprod{u^{\perp}_i,\zeta}$ for each $i \in \set{0,1,\dots,n}$.
\item Repeat rounding $O(\log (1/\eta))$ random choices of $\zeta$ and pick the best choice.
\end{enumerate}
}}
\vspace{-10pt}
\end{center}
\caption{\label{ALG:SDPQP} The SDP-based algorithm for the degree-$2$ optimization problem.}
\vspace{-15pt}

\end{figure}

To prove the theorem we use a semi-definite programming (SDP) based algorithm shown in Figure~\ref{ALG:SDPQP}, that is directly inspired by the SDP-based algorithm for quadratic programming (\textsc{QP}) by~\cite{nesterov1998semidefinite, charikar2004maximizing}. However, the goal in quadratic programming is to find an assignment $x \in \set{-1,1}^n$ that maximizes $\sum_{i \ne j} a_{ij} x_i x_j$. There are three main differences from the QP problem.
Firstly, unlike \textsc{QP} which finds a solution with $\norm{x}_\infty = 1$ with sub-optimal objective value, our goal is to output a solution which attains at least as large a value as $\max_{\norm{x}_\infty \le \delta }g(x)$ while violating the $\ell_\infty$ length of the vector. 
Secondly, unlike QP where the diagonal terms are all $0$, in our problem the diagonal terms can be non-zero and hence it is no longer true that the solution with $\norm{x}_\infty \le 1$ will have each co-ordinate being $\set{\pm 1}$. Finally and most crucially, \textsc{QP} corresponds to optimizing a homogeneous degree $2$ polynomial, with no linear term. 
These challenges necessitates non-trivial modifications to the algorithm and in the analysis. We also remark that it seems unlikely that the upper bound of $O(\sqrt{\log n})$ on the approximation factor can be improved even for the special case of homogenous degree-$2$ polynomials, based on the current state of the approximability of Quadratic Programming (see Remark~\ref{rem:QP} for details).
\anote{Some point about how homogenizing is not obvious.}

The SDP we consider is given by the following equivalent vector program (the SDP variables correspond to $X_{ij}=\iprod{u_i, u_j}$), which can be solved in polynomial time up to arbitrary additive error (using the Ellipsoid algorithm).
\begin{align}
\max_{\set{u_0, u_1, \dots, u_n}} ~~& \sum_{i,j=1}^n A_{ij} \iprod{u_i, u_j}+ \sum_{i=1}^n b_i \iprod{u_i, u_0} + c \label{eq:QP:SDP}\\
\text{s.t.} ~~& \norm{u_i}^2_2 \le \delta^2 ~~\forall i \in \set{1,2,\dots,n},  \text{ and }\norm{u_0}_2^2=1. \label{eq:const2}
\end{align}
Let $\sdpval$ denote the optimal value of the above SDP relaxation.
Clearly the above SDP is a valid relaxation of the problem; for any valid solution $x \in [-\delta, \delta]^n$, consider the solution given by $\big( u_i = x_i u_0: i \in [n])$ for any unit vector $u_0$. Hence $\sdpval \ge \max_{\norm{x}_\infty \le \delta} g(x)$.  Moreover, when the SDP value $\sdpval$ is negative, this certifies that the classifier is robust around the give sample $x^*$. We prove Theorem~\ref{thm:QP:algorithm} by designing a polynomial time rounding algorithm that takes the SDP solution and obtained a valid $\hat{z}$ satisfying the requirements of the theorem.


\vspace{5pt}

\noindent {\bf Rounding Algorithm.} Given the SDP solution, let $u_i^{\perp}$ represent the component of $u_i$ orthogonal to $u_0$. Consider the following randomized rounding algorithm that returns a solution $\set{\widehat{x}_i: i \in [n]}$ :
\begin{align}
\forall i \in \set{0,1,\dots, n},~~\widehat{x}_i & := \iprod{u_i,u_0}+ \iprod{u_i, \zeta} = \iprod{u_i,u_0}+\iprod{u^{\perp}_i,\zeta},
\text{ with } \zeta \sim N\Big(0, \Pi^{\perp}\Big),\label{eq:roundingQP}
\end{align}
where $\Pi^{\perp}$ is the projection matrix onto the subspace of $span(\set{u_1,\dots,u_n})$ that is orthogonal to $u_0$. For convenience, we can assume without loss of generality that $u_0 = e_0$, where $e_0$ is a standard basis vector, and $u_i \in \R^{n+1}$. Let $e_0,e_1, \dots, e_n$ represent an orthogonal basis for $\R^{n+1}$.
Then 
\begin{align*}
\forall i \in \set{0,1,\dots, n},~ \widehat{x}_i & = \iprod{u_i,u_0}+\iprod{u^{\perp}_i,\zeta}
\text{ where } \iprod{\zeta, e_0}=0, ~ \iprod{\zeta, v} \sim N(0,\norm{v}_2^2) \text{ for every } v \perp e_0,
\end{align*}
and $\widehat{x}_0=1$.
The rounding algorithm just tries $O(\log (1/\eta))$ independent random draws for $\zeta$, and picks the best of these solutions.

We now give the analysis of the algorithm.
We prove Theorem~\ref{thm:QP:algorithm} by showing the following guarantee for the rounding algorithm.

\begin{lemma}\label{lem:QP:algorithm}
There is a polynomial time randomized rounding algorithm that takes as input the solution of the SDP as defined in Equations~\ref{eq:QP:SDP}, and \ref{eq:const2}, and outputs a solution $\widehat{x}$ such that
\begin{equation} \label{eq:QP:highprob}
\Pr_{\widehat{x}}\Big[g(\widehat{x}) \ge \max_{\norm{x}_\infty \le \delta} g(x) \text{ and } \norm{\widehat{x}}_\infty \le  O (\sqrt{\log n}) \cdot \delta  \Big] \ge  \Omega(1).
\end{equation}
\end{lemma}
Assuming \eqref{eq:QP:highprob}, we can repeat the algorithm at least $O(\log (1/\eta))$ times to get the guarantee of Theorem~\ref{thm:QP:algorithm}.

\begin{proof}[Proof of Lemma~\ref{lem:QP:algorithm}]
We start with a simple observation that follows from the standard properties of spherical Gaussians. For any $i,j \in [n]$, we have $\E_\zeta[\iprod{u_i^{\perp}, \zeta}\iprod{u_j^{\perp}, \zeta}]=(u_i^{\perp})^T \Pi^{\perp} u_j^{\perp}=\iprod{u_i^{\perp}, u_j^{\perp}}$. Hence we get the key observation that for $\forall i, j \in \set{0,\dots,n}$,
\begin{align}
\E\big[\widehat{x}_i \widehat{x}_j \big] &= \E_\zeta\Big[ \Big(\iprod{u_i,u_0}+\iprod{u_i^{\perp},\zeta} \Big) \Big(\iprod{u_j,u_0}+\iprod{u_j^{\perp},\zeta} \Big) \Big]=\iprod{u_i,u_0} \iprod{u_j,u_0}+ \E_\zeta\Big[\iprod{u_i^{\perp},\zeta}\iprod{ u_j^{\perp},\zeta} \Big] \nonumber
\\
&= \iprod{u_i,u_0} \iprod{u_j,u_0}+ \iprod{u_i^{\perp}, u_j^{\perp}}=\iprod{u_i, u_j}. \label{eq:preserve-expectation}
\end{align}
Note that this also holds when $i=j$. We now consider the expected value of $g(\widehat{x})$. Using \eqref{eq:preserve-expectation}, $\widehat{x}_0=1$ and  since $\E_\zeta[\iprod{u^{\perp}_i,\zeta}]=0$, we have
\begin{align}
\E[g(\widehat{x})]& = \sum_{i,j=1}^n A_{ij} \E_\zeta\Big[\widehat{x}_i \widehat{x}_j \Big]+ \sum_{i=1}^n b_i \E_\zeta[ \widehat{x}_i \widehat{x}_0]+ c \E_\zeta[\widehat{x}_0^2] \nonumber\\
&=  \sum_{i,j=1}^n A_{ij} \iprod{u_i, u_j} + \sum_{i=1}^n b_i \iprod{u_i,u_0} + c \norm{u_0}_2^2 = \sdpval \label{eq:preserve-obj}.
\end{align}

We now show that $\widehat{x}_i \le O(\sqrt{\log n})\cdot\delta$ w.h.p. For each fixed $i \in \set{1,\dots,n}$, $\iprod{u^{\perp}_i,\zeta}$ is distributed as a Gaussian with mean $0$ and variance $\norm{u^{\perp}}_2^2 \le \delta^2$ ,
\begin{align*}
\abs{\widehat{x}_i} &\le \abs{\iprod{u_i , u_0}}+ \abs{\iprod{u^{\perp}_i,\zeta}} \le \delta + \abs{\iprod{u^{\perp}_i,\zeta}} \le \sqrt{C \log n} \cdot \delta\text{ with probability at least } 1- 1/n^{C/2},
\end{align*}
using standard tail properties of Gaussians. Hence, using a union bound over all $i \in [n]$, we have that
\begin{equation} \label{eq:QP:Exp}
\E[g(\widehat{x})] \ge \max_{\norm{x}_\infty \le \delta} g(x), ~~\text{ and } \Pr\Big[\norm{\widehat{x}}_\infty \le O(\sqrt{\log n}) \cdot \delta \Big] \ge 1- \frac{1}{n^2} .
\end{equation}
for $C \ge 4$. Further note that $g(\widehat{x})$ can be expressed a degree-$d$ polynomial of the Gaussian vector $\zeta$.
Hence using hypercontractivity of low-degree polynomials (Theorem 10.23 of \cite{Ryanbook}), we have
\[ \Pr_{\zeta}\Big[ g(\widehat{x})  \ge \E_\zeta g(\widehat{x})\Big]  \ge \Omega(1).\]
Hence  \eqref{eq:QP:highprob} follows.
\end{proof}

\begin{remark}\label{rem:QP}
Obtaining an approximation factor of $O(\gamma)$ in the $\ell_\infty$ norm of $\hat{z}$, even for the special case of homogeneous degree-$2$ polynomials $\sum_{i < j=1}^n a_{ij} x_i x_j$ with no diagonal entries ($a_{ii}=0 ~\forall i \in [n]$) over $\norm{x}_\infty \le \delta$ is equivalent to obtaining a $O(\gamma^2)$-factor approximation algorithm for the problem called Quadratic Programming (QP) which maximizes $\sum_{i<j=1}^n a_{ij} x_i x_j$ over $x \in \set{-1,1}^n$ (this is also called the Grothendieck problem on complete graphs). The best known approximation algorithm for Quadratic Programming (QP) gives an $O(\log n)$-factor approximation in polynomial time~\cite{nesterov1998semidefinite, charikar2004maximizing}. Further~\cite{arora2005non} showed that it is hard to approximate QP within a $O(\log^{c} n)$ for some universal constant $c>0$ assuming NP does not have quasi-polynomial time algorithms. Moreover integrality gaps for SDP relaxations~\cite{alon2006quadratic, khot2006sdp} suggest that $O(\log n)$ factor maybe be tight for polynomial time algorithms. Hence even for the special case of homogeneous degree-$2$ polynomials, improving upon the bound of $\sqrt{\log n}$ in the approximation factor seems unlikely. 
\end{remark}


\section{From Adversarial Examples to Robust Learning Algorithms}
\label{sec:upper-bound}
\anote{Mention how the framework also applies to other norms.}

In this section we will show how to leverage the algorithms for finding adversarial examples to design polynomial time robust learning algorithms for various sub-classes of Polynomial Threshold Functions (PTF). In particular, these include general degree-$1$ and degree-$2$ polynomial threshold functions. We obtain our upper bounds by establishing a general algorithmic framework that relates robust learnability of PTFs to the polynomial maximization problem studied in Section~\ref{sec:optimization}.
This is formalized in the definition below:
\begin{definition}[$\gamma$-factor admissibility]
\label{def:gamma-factor-admissible}
For $\gamma \geq 1$, we say that a sub-class $\mathcal{F}$ of PTFs is $\gamma$-factor admissible if $\mathcal{F}$ has the following properties:
\begin{enumerate}[topsep=0pt,itemsep=0pt,partopsep=1ex,parsep=1ex]
\item For any $a,b,c \in \R, \sgn(f(x)), \sgn(g(x)) \in \mathcal{F}$, $\sgn(a f(x)+b g(x)+c) \in \mathcal{F}$.
\item For any $b \in \mathbb{R}^n$ and $sgn(g(x)) \in \mathcal{F}$, we have that $sgn(g(x+b)) \in \mathcal{F}$.
\item There is a $\gamma$-admissible approximation for $\set{g: sgn(g) \in \mathcal{F}}$.
\end{enumerate}
\end{definition}

The first two conditions above are natural and are satisfied by many sub-classes of PTFs. 
The third condition in the above definition concerns the optimization problem studied in Section~\ref{sec:optimization}. 
The main result of this section, stated below, is the claim that any admissible sub-class of PTFs is also robustly learnable in polynomial time.
\begin{theorem}
\label{thm:upper-bound-main}
Let $\mathcal{F}$ be a sub-class of PTFs that is $\gamma$-factor admissible for $\gamma \geq 1$. Then $\mathcal{F}$ is $\gamma$-approximate robustly learnable.
\end{theorem}
\begin{remark}
While we state our upper bounds for perturbations measured in the $\ell_\infty$ norm, we would like to point out that one can define analogously $\gamma$-factor admissibility for any $\ell_p$ norm and the above theorem will still hold true with the new definition.
\end{remark}
\anote{Some statement about how we leave it as an open question, whether a similar statement holds for neural networks, and with SGD instead of Ellipsoid?}

To learn a $g \in \mathcal{F}$ we formulate robust empirical risk minimization as a convex program, shown in Figure~\ref{ALG:learn}. Here we use the fact that the value of any polynomial $g$ of degree $d$ at a given point $x$ can be expressed as the inner product between the co-efficient vector of $g$ (denoted by $\text{coeff}(g) \in \R^{D}$) and an appropriate vector $\psi(x) \in \R^{D}$ where $D={n+d-1 \choose d}$. \anote{Added the line above.} Our goal is to find a polynomial $g \in \mathcal{F}$ that correctly classifies all the training examples $(x_i, y_i)$. This corresponds to the constraint $y_i g(x_i) >0$ expressed as $y_i \iprod{\text{coeff}(g),\psi(x)}  >0$, a linear constraint in the unknown coefficients $\text{coeff}(g)$ of the polynomial $g$. For example, if $g(x)$ is a degree-$2$ polynomial of the form $x^T A x + b^T x + c$, then the constraint $y_i g(x_i) > 0$ is linear in the unknown coefficients, $a_{i,j}, b_i$ and $c$, of the polynomial. Here $a_{i,j}$ corresponds to the $(i,j)$ entry of the matrix $A$ and $b_i$ is the $i$th coordinate of vector $b$. We also want to ensure that $g$ is robust around each point in the training set. These two constraints together can be enforced by the convex program in Figure~\ref{ALG:learn}, where the $r_i$'s are additional variables apart from the coefficients of $g$. Note that the set of all $g$ is convex because of condition 1 of Definition~\ref{def:gamma-factor-admissible}.
While constraints in \eqref{eq:convex-1} are linear in the variables and easy to implement, \eqref{eq:convex-2} is really asking to check the robustness of $g$ at a given point $(x_i, y_i)$, which is an NP-hard problem~\cite{charikar2004maximizing}. 
Instead, we will use the fact that $\mathcal{F}$ is $\gamma$-factor admissible to design an approximate separation oracle for the type of constraints enforced in \eqref{eq:convex-2}. We would like to mention that the classical literature on robust optimization of linear and convex programs studies a similar setting where typically the goal is to bound the probability of each constraint being violated while achieving the maximum objective value~\cite{ben1999robust, el1997robust, bertsimas2004price}. In contrast, we are interested in precisely quantifying how much a constraint can be violated by and relate the bound to the robustness of the final classifier obtained. We are now ready to prove the main theorem of this section.

\begin{proof}[Proof of Theorem~\ref{thm:upper-bound-main}]
Let $\eta > 0$ be the success probability desired for the robust learning algorithm and $\epsilon > 0$ be the final robust error that is desired. Let $\mathcal{B}$ be an algorithm that achieves the $\gamma$-factor admissibility for the class $\mathcal{F}$. Given $S$, we will run the Ellipsoid algorithm on the convex program in Figure~\ref{ALG:learn}. Let $T(m,n)$ be a (polynomial) upper bound on the number of iterations of the algorithm. In each iteration, given $g, r_1, r_2, \dots, r_m$, we will first check whether $y_i g(x_i) > r_i$. If not, then we have found a violated constraint with the corresponding separating hyperplane being $sgn(r_i - y_i g(x_i))$, and the algorithm proceeds to the next iteration. If all the constraints in \eqref{eq:convex-1} are satisfied, then for each $i \in [m]$, we run $\mathcal{B}$ on the polynomial $y_i(g(x_i) - g(x_i + z))$, where $z$ is the variable and $x_i$ is fixed to be the $i$th data point. Furthermore, we will set $\eta'$, the failure probability of $\mathcal{B}$, to be equal to $\eta/(mT(m,n))$ and set $\delta'$ that is input to $\mathcal{B}$ to be $\delta/\gamma$. From the guarantee of $\mathcal{B}$ we get that if there exists an $i$ such that
\begin{align}
\label{eq:convex-2-approximate}
r_i &< \sup\limits_{z \in B^n_\infty(0,\frac{\delta}{\gamma})} y_i \Big( g(x_i) - g(x_i + z) \Big),
\end{align}
with probability at least $1-\eta/T(m,n)$, the $\mathcal{B}$ will output a violated constraint of the convex program, i.e., an index $i \in [m]$ and $\hat{z} \in B^n_\infty(0,\delta)$ such that
\begin{align*}
r_i &< \sup\limits_{z \in B^n_\infty(0,\delta)} y_i \Big( g(x_i) - g(x_i + \hat{z}) \Big).
\end{align*}

This gives us a separating hyperplane of the form $sgn(y_i(g(x_i) - g(x_i + \hat{z})) - r_i)$, and the algorithm continues. Hence, we get that when the Ellipsoid algorithm terminates, with probability at least $1-\eta$, it will output a polynomial $g \in \mathcal{F}$ such that the constraints in \eqref{eq:convex-1} and \eqref{eq:convex-2-approximate} are satisfied. This means that we would have the empirical robust error $\hat{err}_{\delta/\gamma, S}(sgn(g)) = 0$. Hence, by Lemma~\ref{lem:uniform-convergence}, we get that
$$
err_{\delta/\gamma, D}(sgn(g)) \leq 2\sqrt{\frac{2\Delta \log m}{m}} + \sqrt{\frac{\log \frac{1}{\eta}}{2m}},
$$
where $\Delta$ is the VC dimension of $\mathcal{F}$. Choosing $m = c \frac{\Delta + \log(1/\eta)}{\epsilon^2}$, makes $err_{\delta/\gamma, D}(sgn(g)) \leq \epsilon$.
\end{proof}

\begin{figure}
\begin{center}
\fbox{\parbox{0.98\textwidth}{
\begin{enumerate}[topsep=0pt,itemsep=-1ex,partopsep=1ex,parsep=1ex]
\item Let $S = (x_1, y_1), (x_2, y_2), \dots, (x_m, y_m)$ be the given training set.
\item Find a degree polynomial $g \in \mathcal{F}$ that satisfies
\begin{align}
y_i g(x_i) &> r_i, \,\,\, \forall i\in [m] \label{eq:convex-1}\\
r_i &\geq \sup\limits_{z \in B^n_\infty(0,\delta)} y_i \Big( g(x_i) - g(x_i + z) \Big), \,\,\, \forall i\in [m] \label{eq:convex-2}
\end{align}
\vspace{-10pt}

\end{enumerate}
}}
\end{center}
\caption{\label{ALG:learn} The convex program for finding a polynomial $g \in \mathcal{F}$ with zero robust empirical error.}
\vspace{-10pt}

\end{figure}

It is easy to check that for any fixed $d \in \mathbb{N}$, general degree-$d$ PTFs satisfy conditions $1$ and $2$ of Definition~\ref{def:gamma-factor-admissible} (however homogenous degree $d$ polynomials do not satisfy condition $2$).
We conclude the section by stating the following corollaries about robust learnability of general degree-$1$ and degree-$2$ PTFs.
We begin with the following claim about admissibility and hence robust learnability of degree-$1$ PTFs.
\begin{corollary}
The class of degree-$1$ PTFs is optimally robustly learnable.
\end{corollary}
The proof just follows from Claim~\ref{thm:admissible-degree-1} and since any linear combination or shift of a linear function is also linear.
Similarly, the following corollary about degree-$2$ PTFs is immediate from Theorem~\ref{thm:QP:algorithm}.
\begin{corollary}
The class of degree-$2$ PTFs is $O(\sqrt{\log n})$-approximately robustly learnable.
\end{corollary}

\anote{Can we say something about decoupled PTFs?}


\section{Computational Intractability of Learning Robust Classifiers}
\label{sec:lowerbound}

In this section, we leverage the connection to polynomial optimization to complement our upper bound with the following nearly matching lower bound.
We give a reduction from {\em Quadratic Programming~(QP)} where given a polynomial $p(x)=\sum_{i < j} a_{ij} x_i x_j$, and a value $s$, the goal is to distinguish whether $max_{x \in \{-1,1\}^n }p(x) < s$ or whether exists an $x$ such that $p(x) > s \eta_{approx}$. It is known that the distinguishing problem is hard for $\eta_{approx} = O(\log^c n)$ for some constant $c > 0$~\cite{arora2005non}; moreover the state-of-the-art algorithms give a $\eta_{approx}=O(\log n)$ factor approximation~\cite{charikar2004maximizing} and improving upon this factor is a major open problem. By appropriately scaling the instance, this immediately implies the hardness of checking whether a given degree-$2$ PTF is robust around a given point. 

However, this does not suffice for hardness of learning, since given a distribution supported at a single point, there is 
a trivial constant classifier that robustly classifies the instance correctly. More generally, there could exist a different degree-$2$ PTF that could be easy to certify for the given point. \anote{Added above line} Instead, given a degree-$2$ PTF $\sgn(p(x))$, we carefully construct a set of $O(n^2)$ points such that any classifier that is robust on an instance supported on the set will have to be close to the given polynomial $p$. Having established this, we can distinguish between the two cases of the {\em QP} problem by whether the learning algorithm is able to output a robust classifier or not. This is formalized below.
\anote{Modified below (10/26).}
\begin{theorem}\label{thm:lowerbound-approx}
There exists $\delta,\epsilon>0$, 
such that assuming $NP \ne RP$ there is no algorithm that given a set of $N=poly(n, \frac{1}{\epsilon})$ samples from a distribution $D$ over $\mathbb{R}^n \times \{-1, +1\}$, runs in time $\poly(N)$ and distinguishes between the following two cases for any $\delta'=o(\sqrt{\eta_{approx}} \delta)$:
\begin{itemize}
\item {\sc Yes:} There exists a degree-$2$ PTF that has $\delta$-robust error of $0$ w.r.t. $D$.
\item {\sc No:} There exists no degree-$2$ PTF that has $\delta'$-robust error at most $\epsilon$ w.r.t. $D$.
\end{itemize} 
Here $\eta_{approx}$ is the hardness of approximation factor of the {\em QP} problem.
\end{theorem}
\begin{remark}
\anote{Modified first line.} The above theorem proves that any polynomial time algorithm that always outputs a robust classifier (or declares failure if it does not find one) will have to incur an extra factor of $\Omega(\sqrt{\eta_{approx}})$ in the robustness parameter $\delta$. Our upper bound in Section~\ref{sec:upper-bound} on the other hand matches this bound. 
\anote{Removed the line saying "This is a natural requirement in the context of robust learning algorithms and all our upper bounds in the paper do, in fact, lead to algorithms with this property."}
While our lower bound applies to algorithms that output a classifier of low error, in Appendix (see Theorem~\ref{thm:lowerbound-strong}) we also prove a more robust lower bound that rules out the possibility of an efficient robust learner that incurs an error less than $1/4$.
\end{remark}



We will represent an instance of {\em Quadratic Programming (QP)} by a polynomial $p(x)=x^T A x$ where $A$ is a symmetric matrix with zeros on the diagonal, and $A_{ij}=A_{ji}=a_{ij}/2$.
Formally, the $NP$-hard problem {\em \QP}~\cite{arora2005non,GareyJohnson} is the following: given $\beta>0$ and a polynomial $p(x)=x^T A x$ 
distinguish whether

\textbf{\NO Case} : there exists an assignment $x^* \in \set{-1,1}^n$ such that $p(x^*) >  \beta \eta_{approx}$,

\textbf{\YES Case} : for every assignment $x \in \set{-1,1}^n$, $p(x) < \beta$.

We prove that there exists a $\delta>0$ and a set of $N=\poly(n)$ points such that it is hard to distinguish whether there exists a degree-2 PTF that is $\delta$ robust at all the points or that no degree-2 PTF is $\eta \delta$ robust
for $\eta = \Omega(1/\sqrt{\eta_{approx}})$.

\begin{theorem}\label{THM:lowerbound-approx}[Hardness]
There exists $\delta>0$, such that assuming $NP \ne RP$ there is no polynomial time algorithm that given a set of $N=O(n^{2})$ labeled points $\set{(x^{(1)},y^{(1)}), \dots, (x^{(N)}, y^{(N)})}$ with $(x^{(j)}, y^{(j)}) \in \R^{n+1} \times \set{-1,1}$ for all $j \in [N]$ can distinguish between the following two cases

\textbf{YES Case:} There exists a degree-$2$ PTF that has $\delta$-robust empirical error of $0$ on these $N$ points.

\textbf{NO Case:} No degree-$2$ PTF is $\eta \delta$-robust on these points for $\eta = \Omega(1/\sqrt{\eta_{approx}})$.

\end{theorem}
%

Theorem~\ref{thm:lowerbound-approx} follows from Theorem~\ref{THM:lowerbound-approx} and the standard fact used in establishing learning theoretic hardness~\cite{kearns1994introduction}, namely if there were a robust learning algorithm for every distribution and $\epsilon > 0$, the one could use it on the uniform distribution over the instance from Theorem~\ref{THM:lowerbound-approx} with $\epsilon = \frac{1}{2N}$ to determine whether there exists a degree-$2$ PTF that has $\delta$-robust empirical error of $0$ on the points in the instance.
Hence our main goal is to prove Theorem~\ref{THM:lowerbound-approx}. In order to get hardness of approximation, we need to pick the set of points carefully. Our set of points will have the property that in the YES case of
the QP instance, the polynomial $x^T A x-z$ will be $\delta$ robust at all the points (see Claim~\ref{claim:ugly-proof}). Furthermore, the points will enforce the property that any other degree-$2$ PTF that classifies the points correctly will have to be very close to $x^T A x-z$ in terms of the parameters. This will help use rule out 
the existence of an $\eta \delta$ robust classifier in the NO case, since if one exists, it must be close to $x^T A x - z$, thereby implying an upper bound on the value of $x^T A x$ around the neighborhood of zero. This is established in the following key lemma. 

\begin{lemma}\label{lem:poly-small-coeffs}
Let $p(x,z) = x^T A x - z$ be a given polynomial where $A$ is a symmetric matrix with zeros on the diagonal. For any $\epsilon, \delta < 1/10$, consider the labeled set $S = S_1 \cup S_2 \cup S_3 \cup S_4 \cup S_5$ where,
	$$
	S_1 = \{((\mathbf{0},1),-1), ((\mathbf{0},-1), +1), ((\mathbf{0},\tau'),-1), ((\mathbf{0},-\tau'), +1), ((\mathbf{0},2\delta), -1), ((\mathbf{0},-2\delta), +1)\}, 
	$$
	$$
	S_2 =\{((\mathbf{e}_i, \gamma), -1),((\mathbf{e}_i, -\gamma), +1), ((-\mathbf{e}_i, \gamma), -1), ((-\mathbf{e}_i, -\gamma), +1)\}, \,\, \forall i \in [n],
	$$
	$$
	S_3 = \{((\mathbf{e}_{i,j}, 2), -1),((\mathbf{e}_{-i,j}, 2), -1),((\mathbf{e}_{i,-j}, 2), -1),((\mathbf{e}_{-i,-j}, 2), -1)\}, \,\, \forall i\neq j \in [n],
	$$
	\begin{multline*}
	S_4 = \{((2\mathbf{e}_{i,j}, 1),  \sgn(a_{i,j})),((2\mathbf{e}_{-i,j}, 1), - \sgn(a_{i,j})),((2\mathbf{e}_{i,-j}, 1), - \sgn(a_{i,j})), \\ ((2\mathbf{e}_{-i,-j}, 1),  \sgn(a_{i,j}))\}, \,\, \forall i\neq j \in [n],
	\end{multline*}
	and
	$$
	S_5 = \{((\mathbf{e}_{i,j}, -2), +1),((\mathbf{e}_{-i,j}, -2), +1),((\mathbf{e}_{i,-j}, -2), +1),((\mathbf{e}_{-i,-j}, -2), +1)\}, \,\, \forall i\neq j \in [n],
	$$
	Here $\mathbf{e}_i$ is the vector $(0,0,\dots,\tau,0,\dots,0)$ and $\mathbf{e}_{i,j}$ is the vector $(0,0, \dots, \frac{1}{\sqrt{2(\epsilon + |a_{i,j}|)}}, 0, \dots, \frac{1}{\sqrt{2(\epsilon+ |a_{i,j}|)}}, 0, \dots, 0)$.
	For every general degree 2 polynomial $q'(x,z)$ with the coefficient of $z = c_z$, such that $sgn(q')$ has zero error on $S$, we must have $c_z \neq 0$. Moreover, let $q(x,z) = \frac{1}{-c_z}q'(x,z) = x^T A' x + c^T_1x + c_2 z^2 - z + c_4 + \sum_i \beta_i zx_i$, where $A'$ be a symmetric matrix. Then we must have that
	$$
	\max(|c_2|,  \|\beta\|_\infty, |a'_{i,i}| ) \leq \epsilon,
	$$
	$$
	|c_4| \leq 4 \delta,
	$$
	$$
	|c_{1,i}| \leq \min_{j\neq i} 8 \delta \sqrt{\epsilon + |a_{i,j}|},
	$$
and	
	$$
	\frac 1 4 - \delta - \frac \epsilon 4 \leq \max(\frac{|a'_{i,j}|}{\epsilon + |a_{i,j}|}) \leq 2+4\delta+\epsilon
	$$

	provided $\tau' = \Omega(\frac{n^2}{\epsilon}) \max(1, 1/(\epsilon+\min_{i\neq j} |a_{i,j}|)), \tau = \Omega(\frac{n}{\epsilon}) \max(1, 1/(\epsilon+\min_{i\neq j} |a_{i,j}|))$, $\gamma = 4n \tau$.
\end{lemma}
We first prove Theorem~\ref{THM:lowerbound-approx} assuming the lemma above and finally end the section with the proof of the lemma.
\begin{proof}[Proof of Theorem~\ref{THM:lowerbound-approx}]
Given an $n \times n$ symmetric matrix $A$ with zeros on diagonals and given $s > 100$, we assume that the following cases are hard to distinguish for some $\eta_{approx} > 1$, 

\noindent \textbf{YES Case:} $\max_{x \in \{-1,1\}^n} x^T A x < s$.

\noindent \textbf{NO Case:} $\max_{x \in \{-1,1\}^n} x^T A x > s \eta_{approx}$.
The reduction from the instance of the QP problem is sketched below. Next we establish completeness and soundness of the reduction.
\begin{figure}[H]
	\begin{center}
		\fbox{\parbox{0.98\textwidth}{
				\begin{enumerate}[topsep=0pt,itemsep=0pt,partopsep=1ex,parsep=1ex]
					\item Scale the entries of $A$ such that each non zero entry is greater than 10. Scale $s$ by the same factor. Set $\delta = 1/s$ and $\epsilon = 200/n^2$.
				\item Generate the labeled point set $S$ in $\mathbb{R}^{n+1}$ as specified in Lemma~\ref{lem:poly-small-coeffs} with $\tau' = \Omega(\frac{n^2}{\epsilon}) \max(1, 1/(\epsilon+\min_{i\neq j} |a_{i,j}|)), \tau = \Omega(\frac{n}{\epsilon}) \max(1, 1/(\epsilon+\min_{i\neq j} |a_{i,j}|))$, $\gamma = 4n \tau$. 
				\end{enumerate}
		}}
	\end{center}
	\vspace{-10pt}
	
	\caption{\label{ALG:reduce-hardness-of-approx} Reduction from the QP problem.}
	
	\vspace{-10pt}
\end{figure}

\noindent \textbf{NO Case:} The following claim captures the soundness analysis of the reduction. 
\begin{claim}
There does not exist an $\eta \delta$-robust degree-$2$ polynomial on $S$ for $\eta = \Omega(1/\sqrt{\eta_{approx}})$.
\end{claim}
\begin{proof}
We will prove by contradiction. Let $q(x,z) = x^T A' x + c^T_1x + c_2 z^2 - z + c_4 + \sum_i \beta_i zx_i$ be an $\eta \delta$-robust polynomial on $S$.\footnote{We can always scale $q$ to get it into this form.} The fact that $q$ is correct on $(\mathbf{0}, 2\delta)$ gives us
\begin{align}
\label{eq:hardness-of-approx-1}
4c_2 \delta^2 - 2\delta + c_4 < 0
\end{align}
Furthermore, the fact the fact that $q$ is $\eta \delta$-robust on $(\mathbf{0},2\delta)$ gives us that
\begin{align}
\label{eq:hardness-of-approx-2}
\max_{x \in B^n_\infty(0, \eta \delta), z \in (2\delta - \eta \delta, 2\delta + \eta \delta)} q(x,z)  < |4c_2 \delta^2 - 2\delta + c_4|
\end{align}
From Lemma~\ref{lem:poly-small-coeffs} this implies that
\begin{align}
\label{eq:hardness-of-approx-2}
\max_{x \in B^n_\infty(0, \eta \delta)} x^T A' x  < |4c_2 \delta^2 - 2\delta + c_4| + (2\delta + \eta \delta) + 12 \delta + \epsilon(2\delta + \eta \delta)^2 +  n \epsilon \eta \delta (2\delta + \eta \delta) 
\end{align}
Substituting the value of $\epsilon$ we get that
\begin{align}
\label{eq:hardness-of-approx-3}
\max_{x \in B^n_\infty(0, \eta \delta)} x^T A' x  < 20 \delta.
\end{align}
Again using Lemma~\ref{lem:poly-small-coeffs} we get that
\begin{align}
\label{eq:hardness-of-approx-4}
\max_{x \in B^n_\infty(0, \delta)} x^T A x  < \frac{50 \delta}{\eta^2 }.
\end{align}
But since we are in the NO case we also know that
\begin{align}
\label{eq:hardness-of-approx-5}
\max_{x \in B^n_\infty(0, \delta)} x^T A x  > \delta^2 s \eta_{approx} = \delta \eta_{approx}.
\end{align}
This contradicts the fact that $\eta = \Omega(1/\sqrt{\eta_{approx}})$.
\end{proof}

\noindent \textbf{YES Case:} 
The analysis of the YES case relies on the following claim which establishes $\delta$-robustness of the PTF given by $p(x,z)$ on the point in $S$. 
\begin{claim}\label{claim:ugly-proof}
The polynomial $p(x,z) = x^T A x - z$ is $\delta$-robust on $S$.
\end{claim}

We defer the proofs of Claim~\ref{claim:ugly-proof} and Lemma~\ref{lem:poly-small-coeffs} to Section~\ref{app:lb-proofs}.

\paragraph{A Lower Bound for Weak Robust Learning.}

We also prove a robust lower bound that rules out the possibility of weak robust learning with $\gamma=1$. 
This hardness result allows the algorithm to output a robust classifier that makes errors on constant fraction of the points! Hence, even when there is a degree-$2$ PTF that has $\delta$ robust error of $0$, it is computationally hard to output a degree-$2$ PTF that has $\delta$-robust error of $\epsilon \le \tfrac{1}{4}$.


\begin{theorem}\label{thm:lowerbound-strong}[Stronger Distributional Hardness]
For every $\delta>0$ and $\epsilon \in (0, \frac 1 4)$, assuming $NP \ne RP$ there is no polynomial time algorithm 
that given a set of $N=poly(n, \frac{1}{\epsilon})$ samples from a distribution $D$ over $\mathbb{R}^n \times \{-1, +1\}$ can distinguish between the following two cases:
\begin{itemize}
\item {\sc Yes:} There exists a degree-$2$ PTF that has $\delta$-robust error of $0$ w.r.t. $D$.
\item {\sc No:} There exists no degree-$2$ PTF that has $\delta$-robust error at most $\epsilon$ w.r.t. $D$.
\end{itemize} 

\end{theorem}

The proof of the above theorem uses non-distributional hardness in Theorem~\ref{THM:lowerbound-strong-nondistributional}. Please see Section~\ref{app:weaker-hardness}.

\section{Finding Adversarial Examples for Two Layer Neural Networks}\label{sec:nn}

Next we use the framework in Section~\ref{sec:optimization} to design new algorithms for finding adversarial examples in two layer neural networks with ReLU activations. The description that follows applies to binary classification and can be easily extended to multiclass classification. 
The binary classifier corresponding to the network is $sgn(f_1(x) - f_2(x))= sgn(v^T \sigma(Wx))$ where $v=v_1 - v_2$.
The optimization problem that arises is the following: given an instance with $A \in \R^{m_1 \times n}, \beta \in \R^{m_2}, B \in \R^{m_2 \times n}, c_1 \in \R^n, c_2 \in \R^{m_1}, c_0 \in \R$, the goal is to find $opt(A,B,\beta, c)$, defined as :
\begin{align}
 opt(A,B,\beta, c)&:=\max_{z: \norm{z}_\infty \le \delta} \norm{c_2 + A z}_1 +c_1^T z  - \norm{\beta+Bz}_1 +c_0 \nonumber\\
 &=\max_{z: \norm{z}_\infty \le \delta} \max_{y: \norm{y}_\infty \le 1} y^T A z + c_1^T z +c_2^T y - \sum_{j=1}^{m_2} \abs{\beta_j+ B_j^T z}. \label{eq:nn:opt}
\end{align}
Here $B_j$ is the $j$th row of $B$. Let $c$ denote $(c_0, c_1, c_2)$, and let $opt(A,B,\beta,c)$ be the optimal value of the above problem.
\anote{Motivate this?}

To see the connection to polynomial optimization, notice that if $B=0$, then the above problem is exactly the one we considered in Section~\ref{sec:optimization} in the context of degree-$2$ PTFs. Furthermore, if $A=0$, then~\ref{eq:nn:opt} is a linear program. However, the presence of both the terms involving $A$ and $B$ make~\ref{eq:nn:opt} a challenging optimization problem. Next we discuss how the problem is related to finding adversarial examples for $2$-layer neural networks. A two layer neural network with ReLU gates is given by parameters $(v_1, v_2, W)$ and outputs $f_1(x)=v_1^T \sigma(Wx), f_2(x)=v_2^T \sigma(Wx)$ where $x \in \R^n, v_1, v_2 \in \R^k$ and $W\in \R^{k \times n}$. Here $\sigma: \R^m \to \R^m$ is a co-ordinate wise non-linear operator $\sigma(y)_i = \max\set{0, y_i}$ for each $i \in [m]$. The classifier corresponding to the network is $sgn(f_1(x) - f_2(x))= sgn( (v_1 - v_2)^T \sigma(Wx))= sgn(v^T \sigma(Wx))$. \anote{Motivate this?}

Our algorithm for solving \eqref{eq:nn:opt} given in Figure~\ref{ALG:SDP:nn} is inspired by Algorithm~\ref{ALG:SDPQP} for polynomial optimization. However, the rounding algorithm differs because the variables $y_j$ and variables $z_i$ serve different purposes in \eqref{eq:nn:opt}, and we need to simultaneously satisfy different constraints on them to produce a valid perturbation. 
Moreover when the SDP is negative, then this gives a certificate of robustness around $x$.
\anote{Replaced guarantees by following:}

We remark that one can obtain provable guarantees similar to Theorem~\ref{thm:upper-bound-main} for Algorithm~\ref{ALG:SDP:nn}  under certain regularity conditions about the SDP solution. However, this is unsatisfactory as this depends on the SDP solution to the given instance, as opposed to an explicit structural property of the instance. Obtaining provable guarantees of the latter kind is an interesting open question. The following proposition holds in a more general setting where there can be an extra linear term as described below.
\anote{Should we change constraint $\norm{v_i}^2 \le (1-\eta)^2$ to $\norm{v_i}^2 \le 1$.}
\begin{figure}
\begin{center}
\fbox{\parbox{0.98\textwidth}{
\begin{enumerate}[topsep=0pt,itemsep=-1ex,partopsep=1ex,parsep=1ex]
\item Given instance $I=(A,B,\beta, c)$ of \eqref{eq:nn:opt}, solve SDP with parameter $\eta \in (0,1)$:
\begin{align*}
sdp &= \max \sum_{j \in [m_1],i \in [n]} A_{j,i} \iprod{v_j, u_i} + \sum_{i=1}^n c_1(i) \iprod{u_i,u_0}  + \sum_{j=1}^{m_1} c_2(j) \iprod{u_0,v_j} - \sum_{j \in [m_2]} r_j +c_0 \\
\text{s.t.} & \forall j \in [m_1] ~~ \|v_i\|^2 \leq 1 , ~~ \forall i \in \set{1, \dots, n}~~ \|u_i\|^2 \leq \delta^2, ~~\text{ and } \|u_0\|^2 = 1 \nonumber \\
&\forall  j \in [k_2]~~~ r_j \ge (\beta_j + \sum_j B_{j,i} \iprod{u_i, u_0}),  ~\text{ and } r_j \ge -(\beta_j + \sum_j B_{j,i} \iprod{u_i, u_0}). \nonumber
\end{align*}
\vspace{-10pt}
\item Let $u_i^{\perp}, v_j^{\perp}$ represent the components of $u_i,v_j$ orthogonal to $u_0$. Let $\epsilon \in (0,1)$ with $\epsilon=\Omega(1)/\sqrt{\log m_1}$. Let $\zeta \sim N(0, I)$ be a Gaussian vector; set
$\forall i \in \set{0,1,\dots, n},~\widehat{z}_i  := \iprod{u_i,u_0}+\tfrac{1}{\epsilon}\iprod{u^{\perp}_i,\zeta}, ~ \widehat{y}_j  := \iprod{v_j,u_0}+\epsilon \iprod{v^{\perp}_j,\zeta}
$.
\item Repeat rounding with $poly(n)$ random choices of $\zeta$ and pick the best choice.
\end{enumerate}
}}
\vspace{-10pt}
\end{center}

\caption{\label{ALG:SDP:nn} The SDP-based algorithm for Problem~\eqref{eq:nn:opt}.}

\vspace{-15pt}

\end{figure}

\anote{Remove mention of theoretical guarantees.}

\begin{proposition}
Let $\gamma\ge 1$. Suppose there is an algorithm that given an instance of problem~\eqref{eq:nn:opt} finds a solution $\widehat{z}, \widehat{y}$ with $\norm{\widehat{z}}_\infty \le \gamma \delta, \norm{\widehat{y}}_\infty \le 1$ such that $\widehat{y}^T A\widehat{z} + c_1^T \widehat{z}+c_2^T \hat{y} -\norm{\beta+B \widehat{z}}_1 +c_0 >0$ when $opt(A,B, \beta, ,c)>0$, then
there is a polynomial time algorithm that given a classifier $sgn(f(x))$ corresponding to a two layer neural net where $f(x):=v^T \sigma(Wx)+(v')^Tx$ and an example $x^*$, guarantees to either (a) find an adversarial example in the $\ell_\infty$ ball of $\gamma \delta$ around $x^*$,  or (b) certify the absence of any adversarial example in the $\ell_\infty$ ball of $\delta$.
\end{proposition}
\begin{proof}
Let $\ell(x^*)=sgn(f(x^*))$.
We first observe that $\sigma(y_j)= \tfrac{1}{2}(|y_j|+y_j)$, and $\sigma(Wx)_j= \tfrac{1}{2}(|\iprod{W_j, x}| + \iprod{W_j,x})$, where $W_j$ is the $j$th row of $W$.
We want to find a $\widehat{z}$ with $\norm{\widehat{z}}_\infty \le \gamma \delta$, such that $(-\ell(x^*)) f(x^*+\widehat{z}) >0$, or certify that there is no such $\widehat{z}$ with $\norm{\widehat{z}}_\infty \le \delta$.

Let $S_+=\set{j \in [k]: -\ell(x^*) v_j \ge 0}$ and $S_-=[k]\setminus S_+$ and let $k_1 = |S_+|$. We now split the rows of $W$ into two ($A$ and $B$) as follows: for every $j \in S_+$, define the row $A_j := \tfrac{1}{2}|v_j| W_j$; otherwise let $B_j:=\tfrac{1}{2}|v_j| W_j$. 
\begin{align*}
  -\ell(x^*) f(x^*+z)&=\tfrac{1}{2}\sum_{j \in S_+} |v_j| |\iprod{W_j,x^*+z}| + \tfrac{1}{2}\iprod{v^T W, x^*+z} - \tfrac{1}{2}\sum_{j \in S_-} |v_j| |\iprod{W_j,x^*+z}| \\
  &= \max_{y \in \set{-1,1}^{k_1}} \sum_{j \in S_+}  y_j \iprod{A_j,x^*+z}  - \sum_{j \in S_-}  |\iprod{B_j,x^*+z}| + c_1^T z + c_0,
\end{align*}
where $c_1^T= \tfrac{1}{2}v^T W+ (v')^T$ and $c_0 = \tfrac{1}{2} v^T W x^*$ are constants. Since the dependence on $y$ is linear we also get by substituting $c_2:=A x^*$, $\beta:=B x^*$,
\begin{align*}
  \max_{\norm{z}_\infty \le \delta}  (-\ell(x^*)) f(x^*+z)&= \max_{\norm{z}_\infty \le \delta} \max_{y: \norm{y}_\infty \leq 1} \sum_{j \in S_+}  y_j \iprod{A_j,z} +  c_2^T y+ c_1^T z- \sum_{j \in S_-} |\beta_j+\iprod{B_j, z}| + c_0,
\end{align*}
as required. Now the proposition follows from the same argument as in Proposition~\ref{prop:advexamples}.
\end{proof}

\section{Experiments}
\label{sec:experiments-app}
In this section, we evaluate the performance of the SDP based rounding algorithm outlined in Figure~\ref{ALG:SDP:nn} to generate adversarial examples for depth-2 neural networks with ReLU gates, and compare it with the projected gradient descent(PGD) based attack of Madry et al.~\cite{madry2017towards}. We will show that our approach indeed finds more adversarial examples. This however, comes at a computational cost since we need to solve one SDP per example and per pair of classes. We use the MNIST data set and our two layer neural network has $d=784$ input units, $k=1024$ hidden units and $10$ output units. This leads to an SDP with $d+k+1$ vector variables. On an standard desktop with Intel i5 $4590$ processor, and $4$ cores \@ $3.30$GHz, solving one SDP instance takes $200$ seconds on average. As a consequence we perform our experiments on randomly chosen subsets of the MNIST data set. Another optimization we perform for computational reasons is that given an example $x$ with predicted class $i$, rather than checking for every class $j$, if one can find an attack example $z$ that misclassifies $x+z$ to be in class $j$, we simply pick $j$ to be the class label of the second highest prediction at $x$. Hence, the numbers we report below are an underestimate of the effectiveness of the full SDP based algorithm 

\begin{figure}
	\includegraphics[width = 0.32\textwidth]{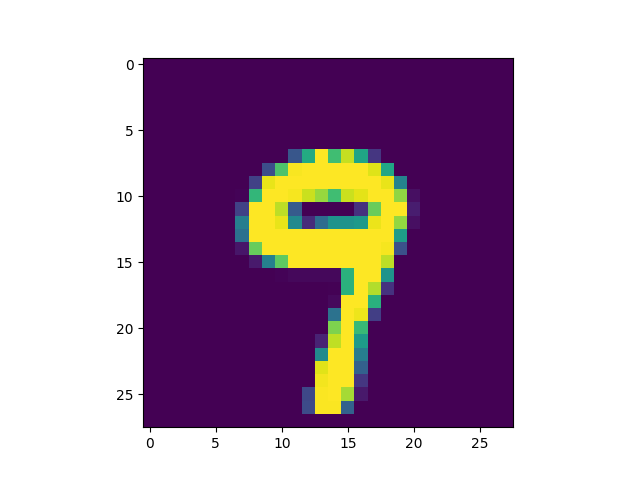}
	\includegraphics[width = 0.32\textwidth]{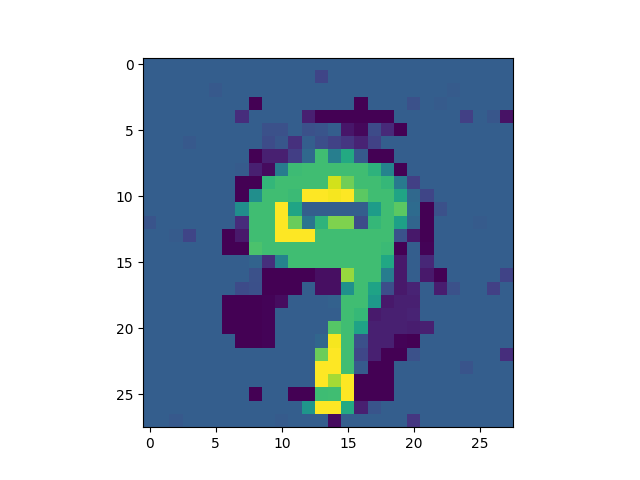}
	\includegraphics[width = 0.32\textwidth]{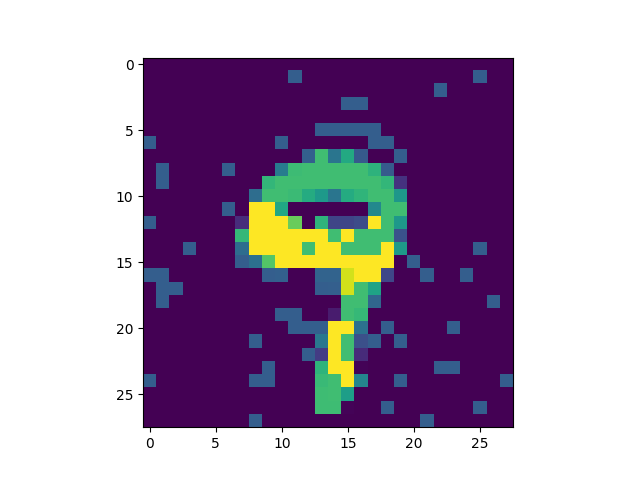}
	\includegraphics[width = 0.32\textwidth]{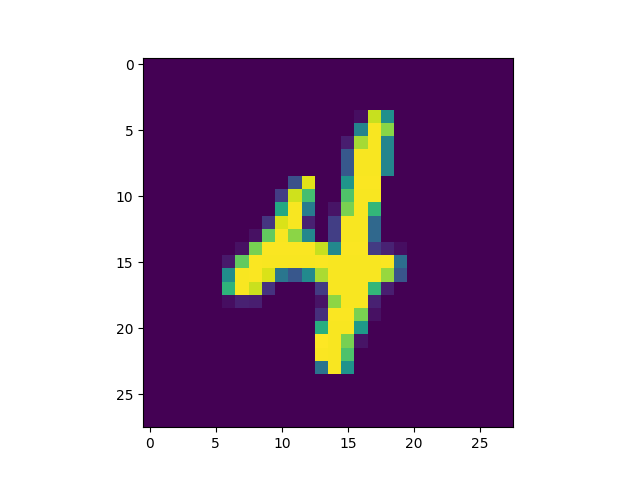}
	\includegraphics[width = 0.32\textwidth]{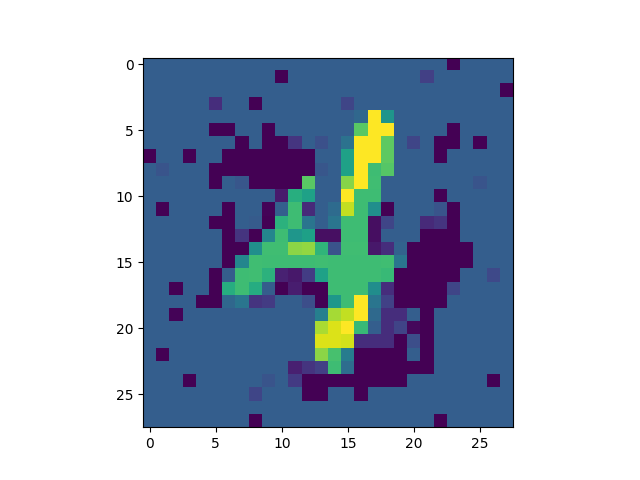}
	\includegraphics[width = 0.32\textwidth]{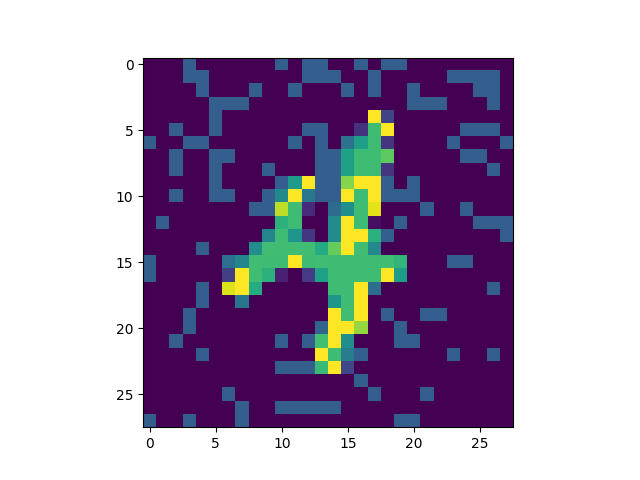}
	\includegraphics[width = 0.32\textwidth]{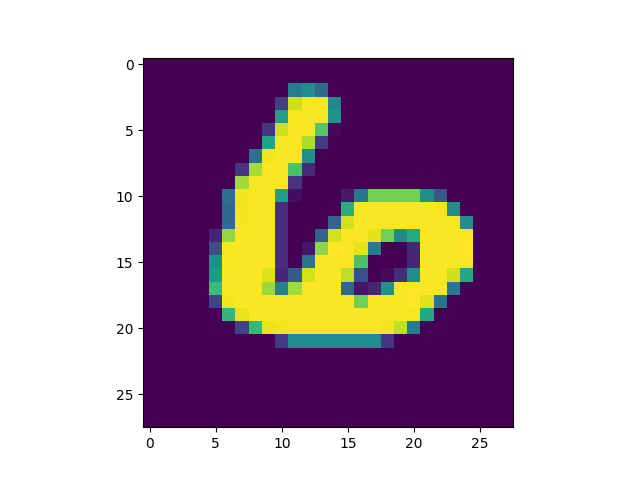}
	\includegraphics[width = 0.32\textwidth]{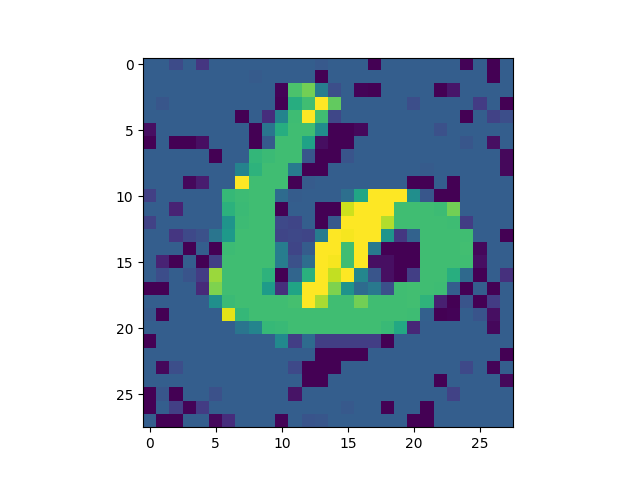}
	\includegraphics[width = 0.32\textwidth]{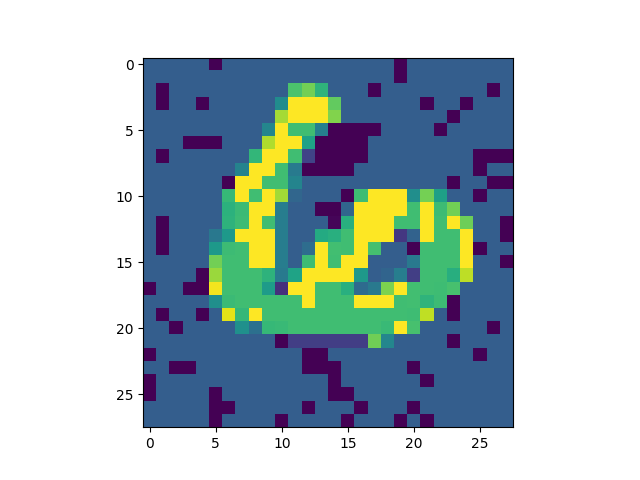}
	\caption{The figure shows three MNIST random samples from PDGfail (i.e., examples where PGDattack failed to find an adversarial perturbation), where SDPattack successfully finds adversarial perturbations for $\delta=0.3$. The images in the first column represent the original images corresponding to three, the second column represents the perturbed images produced by the failed PGDattack, and perturbed images produced by the successful SDPattack. Visual inspection of these examples suggest that our method often produces sparse targeted perturbations. } \label{fig:digits}
\end{figure}

\begin{figure}
	\includegraphics[width = 0.32\textwidth]{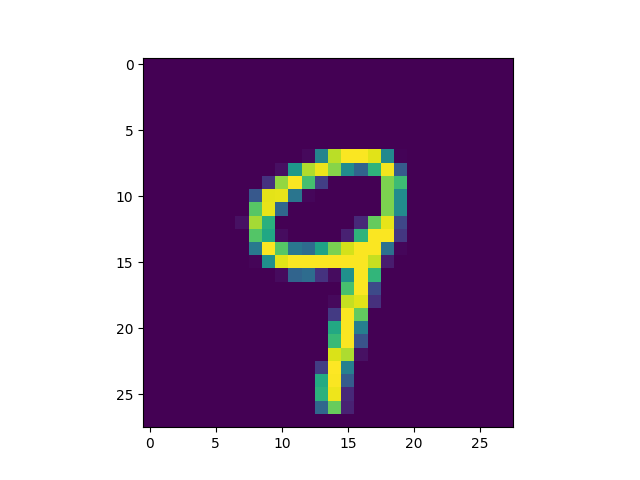}
	\includegraphics[width = 0.32\textwidth]{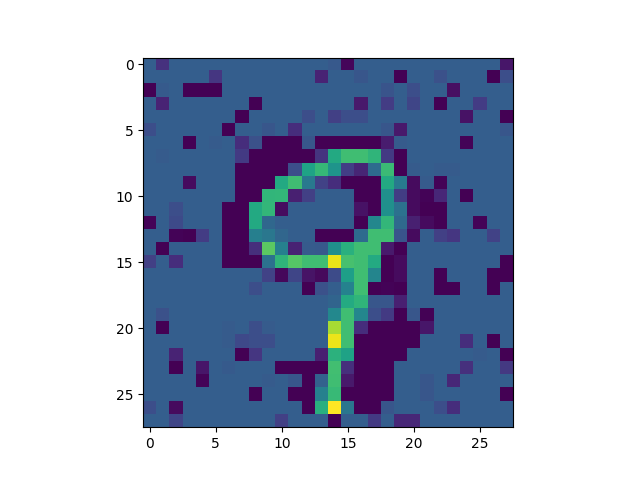}
	\includegraphics[width = 0.32\textwidth]{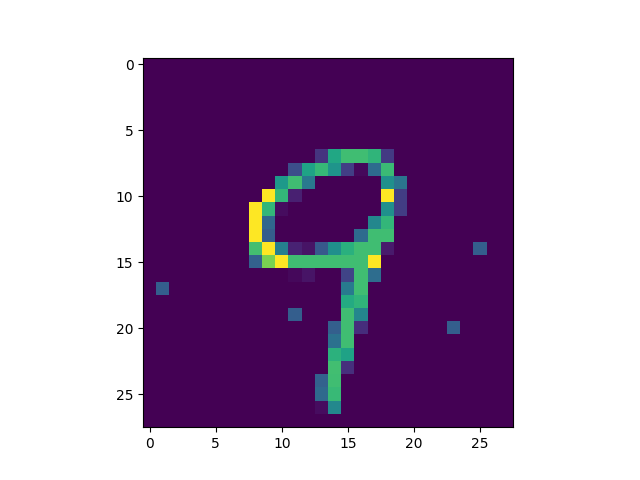}
	\includegraphics[width = 0.32\textwidth]{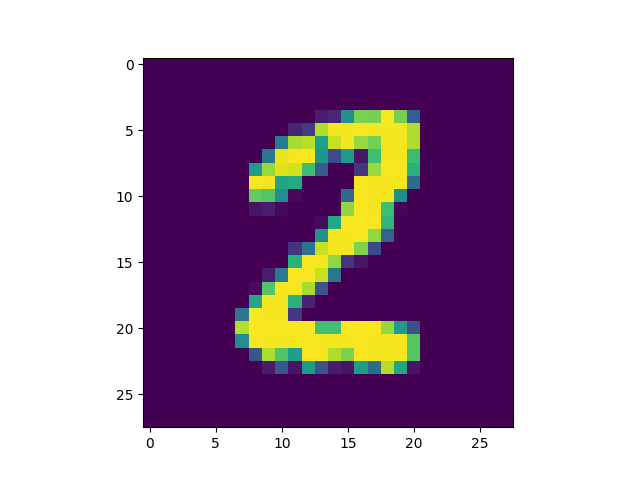}
	\includegraphics[width = 0.32\textwidth]{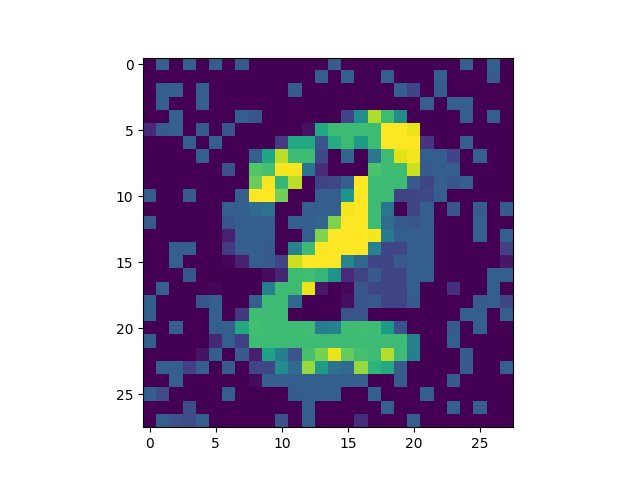}
	\includegraphics[width = 0.32\textwidth]{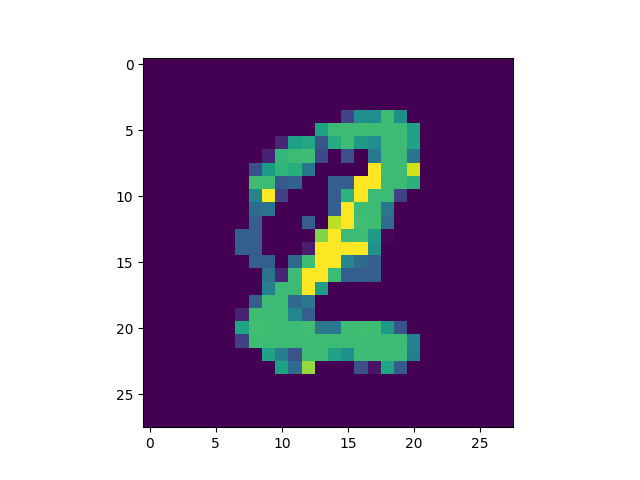}
	\includegraphics[width = 0.32\textwidth]{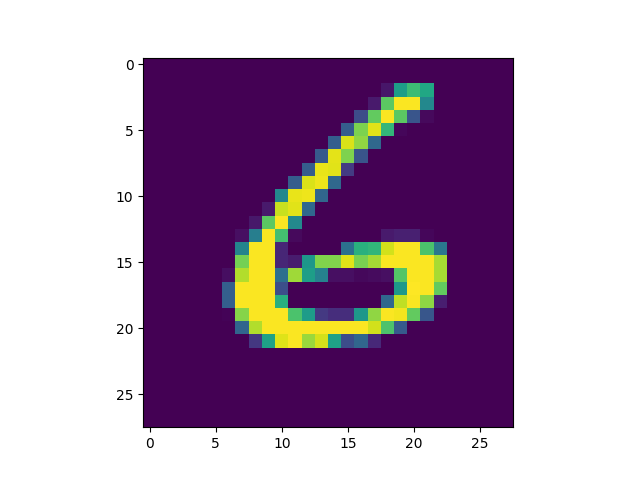}
	\includegraphics[width = 0.32\textwidth]{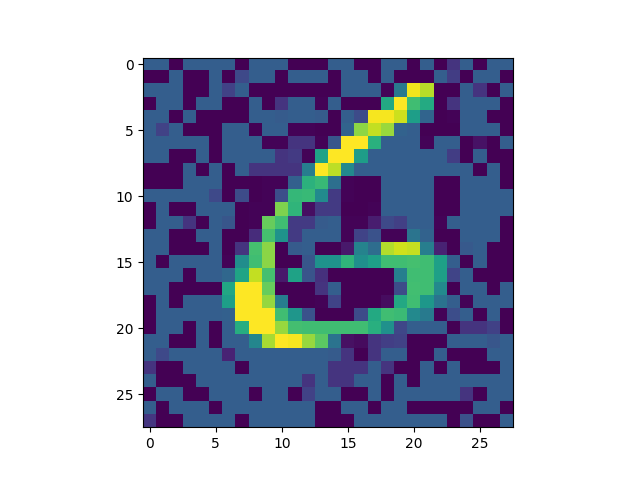}
	\includegraphics[width = 0.32\textwidth]{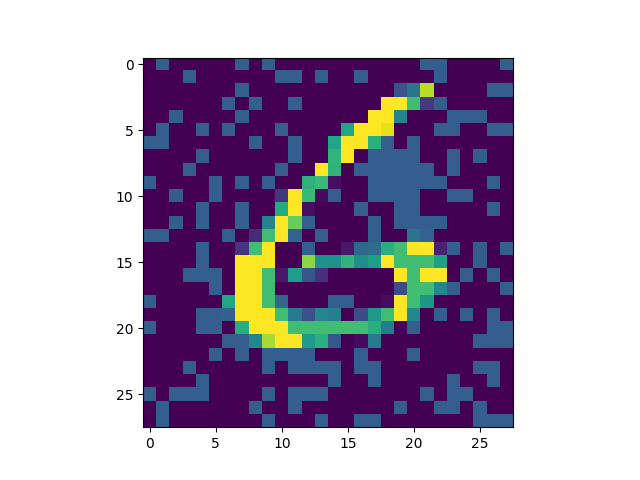}
	\caption{The figure shows three MNIST random samples from PDGpass (i.e., examples where PGDattack succeeded to find an adversarial perturbation), where SDPattack successfully finds adversarial perturbations for $\delta=0.3$. The images in the first column represent the original images corresponding to three, the second column represents the perturbed images produced by the successful PGDattack, and perturbed images produced by the successful SDPattack. Visual inspection of these examples suggest that our method often produces sparse targeted perturbations. } \label{fig:digits-pass}
\end{figure}

We compare the effectiveness of our attack in finding adversarial examples when compared to the the PGD based attack of Madry et al.~\cite{madry2017towards}. We consider two settings of the parameter $\delta$, the maximum amount by which each pixel can be perturbed to produce a valid attack example. As in~\cite{madry2017towards} we first choose $\delta=0.3$ and train a robust $2$-layer network using the algorithm of Madry et al.~\cite{madry2017towards}. We then run the PGD attack and divide the test set into examples where the PGD attack succeeds~(PGDPass) and examples where the PGD attack fails~(PGDfail). We then run our attack on batches of random subsets chosen from each set. In the algorithm we set $\delta' = \alpha \delta$ for a hyperparameter $\alpha \leq 1$. This is because we want to ensure that the
rounded solutions have $\ell_\infty$ norm of at most $\delta$. In our experiment we set $\alpha=0.07$.  The first row of Table~\ref{experiment:table_app} shows the precision and recall of our method. We report the average and the standard deviation across the chosen batches. As one can see, our method has very high recall, i.e., whenever the PGD attack succeeds, our SDP based algorithm also finds adversarial examples. Furthermore, on examples where the PGD attack fails, our method is still able to discover new adversarial examples $30\%$ of the time. Please see Figure~\ref{fig:digits} for the images corresponding to some of the examples where the SDP based attack succeeds, but the PGDattack fails and Figure~\ref{fig:digits-pass} for the images of some examples where both the PGDattack and SDP based attack succeed. A visual inspection of both the figures reveals that our attack often produces sparse targeted attacks as opposed to PGDattack.

We repeat the same methodology with $\delta=0.01, \alpha=0.2$. Here we notice that PGD attack succeeds on only 138 test examples and hence we can afford to run our attack on all of them. As can be seen from the second row of Table~\ref{experiment:table_app} our attack succeeds on all of these examples. Furthermore, we rank the examples in PGDfail according to the difference of the highest and the second highest of the ten network outputs. The smaller the difference, the easier it should be to find an adversarial example. Indeed as can be seen from the table, our method finds 45 new adversarial examples out of the first 100 such ranked examples.








\begin{table}
  \begin{center}
	\begin{tabular}{|c|c|c|c|c|}
		\hline
		$\delta = 0.3$&PGDpass (6 $\times$ 50 random samples)&PGDfail  (8 $\times$ 100 random samples) \\
		\hline
		 SDP succeds&297 out of 300 total& 244 out of 800 total\\
		& Mean : 49.5 of 50, Std : 0.76 &Mean 30.6 of 100, Std : 2.87\\
		\hline
		$\delta = 0.01$&PGDpass (138 samples)&PGDfail  (100 ranked)\\
		\hline
		SDP succeeds&138 & 45\\
		\hline
	\end{tabular}
  \end{center}
  \caption{\label{experiment:table_app} For $\delta = 0.3$, we report mean and standard deviation of number of adversarial examples found by running our SDPattack algorithm on 6 batches of 50 random examples from PGDpass and 8 batches of 100 random samples from PGDfail. For $\delta = 0.01$, we run SDPattack on all 138 examples in PGDpass and first 100 sorted examples from PGDfail.}
\end{table}

The experiments above suggest that our algorithms can lead to improved adversarial attacks.
We would like to note that the recent work of~\cite{raghunathan2018certified} also studied semi-definite programming based methods for
providing adversarial certificates for $2$-layer neural networks. However, our SDP as outlined is Figure~\ref{ALG:SDP:nn} is strictly stronger. In particular, the SDP of~\cite{raghunathan2018certified} is independent of the given example $x$ and as a result we expect our method to produce better certificates. We leave as future work the task of making our theoretical analysis practical for large scale applications.

\dnote{unclipped SDP}

\textbf{}


\section{Future Directions} 
\label{sec:conclusions}

Design of polynomial time algorithms that provably achieve adversarial robustness is an important direction of research. Several open questions remain to be explored further. In Section~\ref{sec:upper-bound} we provide a general algorithmic framework for designing polynomial time robust algorithms. It would be interesting to use our framework to design robust algorithms for general degree-$d$ PTFs. While there are algorithms to approximately maximize degree-$d$ polynomials, they focus on the homogeneous case which does not suffice for our purposes. 
Another important direction for future work is to convert our adversarial attack algorithm for $2$-layer neural networks into a provably robust learning algorithm via the framework of Section~\ref{sec:upper-bound}. A straightforward invocation of the framework does not lead to a convex constraint set. It would also be interesting to design
provable adversarial attacks for higher depth networks. 
Finally, our experimental results suggest that making our SDP based attack work on a large scale could lead to improved adversarial attacks.

\subsection*{Acknowledgements}
The second and third authors were supported by the National Science Foundation (NSF) under Grant No.~CCF-1652491 and CCF-1637585. Additionally, the second author was funded by the Morrison Fellowship from Northwestern University.

\bibliographystyle{plain}
\bibliography{paper}
\appendix

\section{Proofs for the Computational Hardness Results}

\subsection{Proofs of Claim~\ref{claim:ugly-proof} and Lemma~\ref{lem:poly-small-coeffs}} \label{app:lb-proofs}

We first prove Claim~\ref{claim:ugly-proof} that helps establish the YES case. 
\paragraph{Proof of Claim~\ref{claim:ugly-proof}.}
\begin{proof}
It is easy to check that $sgn(x^T A x - z)$ classifies all of $S$ correctly.

\noindent \textbf{Robustness at $((\mathbf{0}, 2\delta), -1)$}. Follows from the fact that we are in the YES case and hence $\max_{x \in B^n_\infty(0, \delta)} x^T A x < \delta^2 s = \delta$.

\noindent \textbf{Robustness at $((\mathbf{0}, 1), -1), ((\mathbf{0}, \tau'), -1), ((\mathbf{0}, -1), +1), ((\mathbf{0}, -\tau'), +1)$}. Follows from the fact that we are in the YES case and hence $\max_{x \in B^n_\infty(0, \delta)} x^t A x < \delta^2 s = 1/s  < 1/100$ and that $\tau'  > n/(20 \delta) > 5n$.

\noindent \textbf{Robustness at $((\mathbf{0},2\delta), -1), ((\mathbf{0},-2\delta), +1)$}. Follows from the fact that we are in the YES case and hence $\max_{x \in B^n_\infty(0, \delta)} x^T A x < \delta^2 s = \delta$ and that $\epsilon n/10 = 20\delta$.

\noindent \textbf{Robustness at $((\mathbf{e}_i, \gamma), -1),((\mathbf{e}_i, -\gamma), +1), ((-\mathbf{e}_i, \gamma), -1), ((-\mathbf{e}_i, -\gamma), +1)$}. Let's argue robustness at $((\mathbf{e}_i, \gamma), -1)$ and the other calculations are similar. The maximum value of $x^T A x$ in a $\delta$-ball around $\mathbf{e}_i$ is at most 
$$
(\tau+\delta)\delta \sum_{j} |a_{i,j}| + \delta^2 s.
$$
Hence to establish robustness we need that
\begin{align}
(\tau+\delta)\delta \sum_{j} |a_{i,j}| + \delta^2 s \leq \gamma - \delta.
\end{align}
Substituting the value of $\delta$ and noticing that $\gamma, \tau$ are much larger than $\delta = 1/s < 1/100$ we get that it is enough for the following to hold
\begin{align}
2 \tau \delta \sum_{j} |a_{i,j}| \leq \frac{\gamma}{2}.
\end{align}
In other words we need that
\begin{align}
\frac{\gamma}{\tau} \geq 4\delta \sum_j |a_{i,j}|
\end{align}
Substituting the values of $\gamma, \tau$ we get that
\begin{align}
n \geq \delta \sum_j |a_{i,j}|
\end{align}
This is true since $\delta =1/s$ and the fact that $s \geq \frac{1}{n} \sum_{i,j} |a_{i,j}| > \frac{1}{n} \sum_{j} |a_{i,j}|$ where the first inequality is from~\cite{charikar2004maximizing}.

\noindent \textbf{Robustness at $((\mathbf{e}_{i,j}, 2), -1),((\mathbf{e}_{-i,j}, 2), -1),((\mathbf{e}_{i,-j}, 2), -1),((\mathbf{e}_{-i,-j}, 2), -1)$}. Let's argue robustness at $((\mathbf{e}_{i,j}, 2), -1)$ and the other calculations are similar. The maximum value of $x^T A x$ in a $\delta$-ball around $\mathbf{e}_{i,j}$ is at most 

$$\frac{2\delta \max_i \sum_j |a_{i,j}|}{\sqrt{2(\epsilon+|a_{i,j}|)}} + \delta^2 s + 1$$

Hence to establish robustness we need that
\begin{align}
\frac{2\delta \max_i \sum_j |a_{i,j}|}{\sqrt{2(\epsilon+|a_{i,j}|)}} + \delta^2 s + 1 \leq 2 - \delta.
\end{align}
Noticing that $\delta = 1/s$ and much smaller than $1/100$, we get that it is enough for the following to hold
\begin{align}
\frac{\delta \max_i \sum_j |a_{i,j}|}{\sqrt{2(\epsilon+|a_{i,j}|)}} \leq \frac 1 4.
\end{align}
This is again true since $\delta = 1/s$ and by our assumption $|a_{i,j}| \geq 4$ for non-zero entries of $A$.
\end{proof}

\noindent \textbf{Robustness at $((2\mathbf{e}_{i,j}, 1), \sgn(a_{i,j})),((2\mathbf{e}_{-i,j}, 1), - \sgn(a_{i,j})),((2\mathbf{e}_{i,-j}, 1), - \sgn(a_{i,j})), ((2\mathbf{e}_{-i,-j}, 1),  \sgn(a_{i,j}))$}. We'll argue robustness at $((2\mathbf{e}_{i,j}, 1), +1)$ and the other calculations are similar. Also for simplicity, assume $\sgn(a_{i,j} >  0$. The other case is similar. The minimum value of $x^T A x$ in a $\delta$-ball arond $\mathbf{e}_{i,j}$ is at least

$$2 - \frac{2\delta \max_i \sum_j |a_{i,j}|}{\sqrt{2(\epsilon+|a_{i,j}|)}} - \delta^2s$$

So for robustness, we need 

$$2 - \frac{2\delta \max_i \sum_j |a_{i,j}|}{\sqrt{2(\epsilon+|a_{i,j}|)}} - \delta^2s  > 1 + \delta$$

This is true since we have $$\frac{\delta \max_i \sum_j |a_{i,j}|}{\sqrt{2(\epsilon+|a_{i,j}|)}} \leq \frac 1 4.$$
\end{proof}

\paragraph{Proof of Lemma~\ref{lem:poly-small-coeffs}.}
We now prove the key lemma that is used in the analysis of our reduction. 
\begin{proof}
	
	First we prove that if $q'(x,z)$ has zero error on $S$ then $c_z$ must be non zero. Then it is clear that if $q'(x,z)$ has zero error on $S$, then so does $q(x,z)$. Consider the case when $c_z = 0$. Now $q'(x,z)$ classifies $S_1$ correctly. More specifically, it classifies the two points $((\mathbf{0},1),-1)$ and $((\mathbf{0},-1),1)$ correctly. This gives us the following equations $$c_2 + c_4 < 0$$ $$c_2 + c_4 > 0$$ and hence we get a contradiction. 
Moving on to the main part of the proof about the coefficients of $q(x,z)$, the constraints at $(\mathbf{0},1), (\mathbf{0},-1), (\mathbf{0},\tau'), (\mathbf{0},-\tau')$ give us
	\begin{align}
		\label{eq:small-coeff-1}
		c_2-1+c_4 < 0
	\end{align}
	\begin{align}
		\label{eq:small-coeff-2}
		c_2+1+c_4 > 0
	\end{align}
	\begin{align}
		\label{eq:small-coeff-3}
		\tau'^2 c_2-\tau'+c_4 < 0
	\end{align}
	\begin{align}
		\label{eq:small-coeff-4}
		\tau'^2 c_2+\tau'+c_4 > 0
	\end{align}
	From (\ref{eq:small-coeff-1}) and (\ref{eq:small-coeff-2}) we get that
	\begin{align}
		\label{eq:small-coeff-5}
		-1 < c_2+c_4 <1
	\end{align}
	Similarly, from (\ref{eq:small-coeff-3}) and (\ref{eq:small-coeff-4}) we get that
	\begin{align}
		\label{eq:small-coeff-6}
		-\tau' < \tau'^2c_2+c_4 <\tau'
	\end{align}
	This implies that $|c_2| < 1/(\tau'-1) < \epsilon/10$ for $\tau' = \Omega(1/\epsilon)$. 
	
	The constraints at $((\mathbf{0},2\delta), -1), ((\mathbf{0},-2\delta)$ gives us that
\begin{align*}
4c_2 \delta^2 - 2\delta + c_4&< 0\\
4c_2 \delta^2 + 2\delta +c_4 &>0
\end{align*}
From the above equations we get that
	\begin{align}
		\label{eq:small-coeff-23}
		|c_4| \leq c_2 (2 \delta)^2 + 2 \delta < 4 \delta.
	\end{align}
	
	The constraints at $(\mathbf{e}_i,\gamma), (-\mathbf{e}_i,\gamma), (\mathbf{e}_i,-\gamma), (-\mathbf{e}_i,-\gamma)$ give us
	\begin{align}
		\label{eq:small-coeff-7}
		\tau^2a'_{i,i} + \tau c_{1,i} + c_2\gamma^2-\gamma + c_4 + \tau \gamma \beta_i< 0
	\end{align}
	\begin{align}
		\label{eq:small-coeff-8}
		\tau^2a'_{i,i} - \tau c_{1,i} + c_2\gamma^2 - \gamma + c_4 - \tau \gamma \beta_i < 0
	\end{align}
	\begin{align}
		\label{eq:small-coeff-9}
		\tau^2a'_{i,i} + \tau c_{1,i} + c_2 \gamma^2 + \gamma + c_4 - \tau \gamma \beta_i > 0
	\end{align}
	\begin{align}
		\label{eq:small-coeff-10}
		\tau^2 a'_{i,i} - \tau c_{1,i} + c_2 \gamma^2 + \gamma + c_4 + \tau \gamma \beta_i > 0
	\end{align}
	From (\ref{eq:small-coeff-7}) and (\ref{eq:small-coeff-10}) we get that
	\begin{align}
		\label{eq:small-coeff-11}
		\tau c_{1,i} < \gamma
	\end{align}
	Similarly, from (\ref{eq:small-coeff-8}) and (\ref{eq:small-coeff-9}) we get that
	\begin{align}
		\label{eq:small-coeff-12}
		\tau c_{1,i} > -\gamma
	\end{align}
	Plugging back into the equations above we get that
	\begin{align}
		\label{eq:small-coeff-21}
		-(4\delta+2\gamma +  \frac{\gamma^2}{\tau'-1} ) < \tau^2 a'_{i,i} + \tau \gamma \beta_i < (4\delta+2\gamma +  \frac{\gamma^2}{\tau'-1} )
	\end{align}
	and
	\begin{align}
		\label{eq:small-coeff-22}
		-(4\delta+2\gamma + \frac{\gamma^2}{\tau'-1} ) < \tau^2 a'_{i,i} - \tau \gamma \beta_i < (4\delta+2\gamma + \frac{\gamma^2}{\tau'-1} )
	\end{align}
	This implies that 
	$$
	|a'_{i,i}| \leq \frac{1}{\tau^2}(4\delta+2\gamma + \frac{\gamma^2}{\tau'-1} ) \leq \epsilon/10
	$$ 
	for $\tau' = \Omega(\frac{n^2}{\epsilon}) \max(1, 1/\min_{i,j} |a_{i,j}|), \tau = \Omega(\frac{n}{\epsilon}) \max(1, 1/\min_{i,j} |a_{i,j}|)$, $\gamma = 4n \tau$. We also get that
	$$
	|\beta_i| \leq \frac{1}{\tau \gamma}(4\delta+2\gamma + \frac{\gamma^2}{\tau'-1} ) \leq \epsilon/10
	$$ 
	for $\tau' = \Omega(\frac{n^2}{\epsilon}) \max(1, 1/\min_{i,j} |a_{i,j}|), \tau = \Omega(\frac{n}{\epsilon}) \max(1, 1/\min_{i,j} |a_{i,j}|)$, $\gamma = 4n \tau$.
	
	The constraints at $(\mathbf{e}_{i,j},2), (\mathbf{e}_{-i,j},2), (\mathbf{e}_{i,-j},2), (\mathbf{e}_{-i,-j},2)$ give us
	\begin{align}
		\label{eq:small-coeff-13}
		\frac{a'_{i,i}}{2\tilde{a}_{i,j}} + \frac{a'_{j,j}}{2\tilde{a}_{i,j}} + \frac{a'_{i,j}}{\tilde{a}_{i,j}} + \frac{c_{1,i}}{\sqrt{2\tilde{a}_{i,j}}} + \frac{c_{1,j}}{\sqrt{2\tilde{a}_{i,j}}} + 4c_2-2+c_4 + \frac{2\beta_i}{\sqrt{2\tilde{a}_{i,j}}} + \frac{2\beta_j}{\sqrt{2\tilde{a}_{i,j}}} < 0
	\end{align}
	\begin{align}
		\label{eq:small-coeff-14}
		\frac{a'_{i,i}}{2\tilde{a}_{i,j}} + \frac{a'_{j,j}}{2\tilde{a}_{i,j}} - \frac{a'_{i,j}}{\tilde{a}_{i,j}} - \frac{c_{1,i}}{\sqrt{2\tilde{a}_{i,j}}} + \frac{c_{1,j}}{\sqrt{2\tilde{a}_{i,j}}} + 4c_2-2+c_4 - \frac{2\beta_i}{\sqrt{2\tilde{a}_{i,j}}} + \frac{2\beta_j}{\sqrt{2\tilde{a}_{i,j}}} < 0
	\end{align}
	\begin{align}
		\label{eq:small-coeff-15}
		\frac{a'_{i,i}}{2\tilde{a}_{i,j}} + \frac{a'_{j,j}}{2\tilde{a}_{i,j}} - \frac{a'_{i,j}}{\tilde{a}_{i,j}} + \frac{c_{1,i}}{\sqrt{2\tilde{a}_{i,j}}} - \frac{c_{1,j}}{\sqrt{2\tilde{a}_{i,j}}} + 4c_2-2+c_4 + \frac{2\beta_i}{\sqrt{2\tilde{a}_{i,j}}} - \frac{2\beta_j}{\sqrt{2\tilde{a}_{i,j}}} < 0
	\end{align}
	\begin{align}
		\label{eq:small-coeff-16}
		\frac{a'_{i,i}}{2\tilde{a}_{i,j}} + \frac{a'_{j,j}}{2\tilde{a}_{i,j}} + \frac{a'_{i,j}}{\tilde{a}_{i,j}} - \frac{c_{1,i}}{\sqrt{2\tilde{a}_{i,j}}} - \frac{c_{1,j}}{\sqrt{2\tilde{a}_{i,j}}} + 4c_2-2+c_4 - \frac{2\beta_i}{\sqrt{2\tilde{a}_{i,j}}} - \frac{2\beta_j}{\sqrt{2\tilde{a}_{i,j}}} < 0
	\end{align}
where $\tilde{a}_{i,j} = \epsilon + |a_{i,j}|$.
	Combining (\ref{eq:small-coeff-13}) and (\ref{eq:small-coeff-16}) we get
	\begin{align}
		\label{eq:small-coeff-17}
		\frac{a'_{i,i}}{2\tilde{a}_{i,j}} + \frac{a'_{j,j}}{2\tilde{a}_{i,j}} + \frac{a'_{i,j}}{\tilde{a}_{i,j}} + 4c_2 - 2 + c_4 < 0
	\end{align}
	From this we get that
	\begin{align}
		\label{eq:small-coeff-18}
		\frac{a'_{i,j}}{\tilde{a}_{i,j}} < 2+4\delta+4\frac{\epsilon}{10} + \frac{4\delta+2\gamma + \frac{\gamma^2}{\tau'-1}}{\tau^2 \min_{i,j} |a_{i,j}|}  < 2+4\delta+\epsilon
	\end{align}
	for large enough $\tau$.
	Similarly, combining (\ref{eq:small-coeff-14}) and (\ref{eq:small-coeff-15}) we get
	\begin{align}
		\label{eq:small-coeff-19}
		\frac{a'_{i,i}}{2\tilde{a}_{i,j}} + \frac{a'_{j,j}}{2\tilde{a}_{i,j}} - \frac{a'_{i,j}}{\tilde{a}_{i,j}} + 4c_2 - 2 + c_4 < 0
	\end{align}
	From this we get that
	\begin{align}
		\label{eq:small-coeff-20}
		\frac{a'_{i,j}}{\tilde{a}_{i,j}} > -2 - 4\delta - \epsilon.
	\end{align}

	The constraints at $(\mathbf{e}_{i,j},-2), (\mathbf{e}_{-i,j},-2), (\mathbf{e}_{i,-j},-2), (\mathbf{e}_{-i,-j},-2)$ give us
	\begin{align}
		\label{eq:small-coeff-43}
		\frac{a'_{i,i}}{2\tilde{a}_{i,j}} + \frac{a'_{j,j}}{2\tilde{a}_{i,j}} + \frac{a'_{i,j}}{\tilde{a}_{i,j}} + \frac{c_{1,i}}{\sqrt{2\tilde{a}_{i,j}}} + \frac{c_{1,j}}{\sqrt{2\tilde{a}_{i,j}}} + 4c_2 + 2+c_4 + \frac{2\beta_i}{\sqrt{2\tilde{a}_{i,j}}} + \frac{2\beta_j}{\sqrt{2\tilde{a}_{i,j}}} > 0
	\end{align}
	\begin{align}
		\label{eq:small-coeff-44}
		\frac{a'_{i,i}}{2\tilde{a}_{i,j}} + \frac{a'_{j,j}}{2\tilde{a}_{i,j}} - \frac{a'_{i,j}}{\tilde{a}_{i,j}} - \frac{c_{1,i}}{\sqrt{2\tilde{a}_{i,j}}} + \frac{c_{1,j}}{\sqrt{2\tilde{a}_{i,j}}} + 4c_2 + 2+c_4 - \frac{2\beta_i}{\sqrt{2\tilde{a}_{i,j}}} + \frac{2\beta_j}{\sqrt{2\tilde{a}_{i,j}}} > 0
	\end{align}
	\begin{align}
		\label{eq:small-coeff-45}
		\frac{a'_{i,i}}{2\tilde{a}_{i,j}} + \frac{a'_{j,j}}{2\tilde{a}_{i,j}} - \frac{a'_{i,j}}{\tilde{a}_{i,j}} + \frac{c_{1,i}}{\sqrt{2\tilde{a}_{i,j}}} - \frac{c_{1,j}}{\sqrt{2\tilde{a}_{i,j}}} + 4c_2 + 2+c_4 + \frac{2\beta_i}{\sqrt{2\tilde{a}_{i,j}}} - \frac{2\beta_j}{\sqrt{2\tilde{a}_{i,j}}} > 0
	\end{align}
	\begin{align}
		\label{eq:small-coeff-46}
		\frac{a'_{i,i}}{2\tilde{a}_{i,j}} + \frac{a'_{j,j}}{2\tilde{a}_{i,j}} + \frac{a'_{i,j}}{\tilde{a}_{i,j}} - \frac{c_{1,i}}{\sqrt{2\tilde{a}_{i,j}}} - \frac{c_{1,j}}{\sqrt{2\tilde{a}_{i,j}}} + 4c_2 + 2+c_4 - \frac{2\beta_i}{\sqrt{2\tilde{a}_{i,j}}} - \frac{2\beta_j}{\sqrt{2\tilde{a}_{i,j}}} > 0
	\end{align}
	Combining (\ref{eq:small-coeff-13}) and (\ref{eq:small-coeff-44}) we get
	\begin{align}
		\label{eq:small-coeff-47}
		\frac{a'_{i,j}}{\tilde{a}_{i,j}}  + \frac{c_{1,i}}{\sqrt{2\tilde{a}_{i,j}}} - 2 + \frac{2 \beta_i}{\sqrt{2\tilde{a}_{i,j}}} < 0
	\end{align}
	From this we get that
	\begin{align}
		\label{eq:small-coeff-48}
		c_{1,i} < (4\delta + \epsilon)\sqrt{2\tilde{a}_{i,j}} 
	\end{align}
	for large enough $\tau$.
	Similarly, from (\ref{eq:small-coeff-45}) and (\ref{eq:small-coeff-15}) we get
	\begin{align}
		\label{eq:small-coeff-49}
		c_{1,i} > -(4\delta + \epsilon)\sqrt{2\tilde{a}_{i,j}}.
	\end{align}

Finally, the constraints at $(2\mathbf{e}_{i,j},1), (2\mathbf{e}_{-i,j},1), (2\mathbf{e}_{i,-j},1), (2\mathbf{e}_{-i,-j},1)$ give us
	\begin{align}
		\label{eq:small-coeff-33}
		2\frac{a'_{i,i}}{\tilde{a}_{i,j}} + 2\frac{a'_{j,j}}{\tilde{a}_{i,j}} + 4\frac{a'_{i,j}}{\tilde{a}_{i,j}} + \frac{2c_{1,i}}{\sqrt{2\tilde{a}_{i,j}}} + \frac{2c_{1,j}}{\sqrt{2\tilde{a}_{i,j}}} + c_2-1+c_4 + \frac{4\beta_i}{\sqrt{2\tilde{a}_{i,j}}} + \frac{4\beta_j}{\sqrt{2\tilde{a}_{i,j}}} > 0
	\end{align}
	\begin{align}
		\label{eq:small-coeff-34}
		2\frac{a'_{i,i}}{\tilde{a}_{i,j}} + 2\frac{a'_{j,j}}{\tilde{a}_{i,j}} - 4\frac{a'_{i,j}}{\tilde{a}_{i,j}} - \frac{2c_{1,i}}{\sqrt{2\tilde{a}_{i,j}}} + \frac{2c_{1,j}}{\sqrt{2\tilde{a}_{i,j}}} + c_2-1+c_4 - \frac{4\beta_i}{\sqrt{2\tilde{a}_{i,j}}} + \frac{4\beta_j}{\sqrt{2\tilde{a}_{i,j}}} < 0
	\end{align}
	\begin{align}
		\label{eq:small-coeff-35}
		2\frac{a'_{i,i}}{\tilde{a}_{i,j}} + 2\frac{a'_{j,j}}{\tilde{a}_{i,j}} - 4\frac{a'_{i,j}}{\tilde{a}_{i,j}} + \frac{2c_{1,i}}{\sqrt{2\tilde{a}_{i,j}}} - \frac{2c_{1,j}}{\sqrt{2\tilde{a}_{i,j}}} + c_2-1+c_4 + \frac{4\beta_i}{\sqrt2{\tilde{a}_{i,j}}} - \frac{4\beta_j}{\sqrt{2\tilde{a}_{i,j}}} < 0
	\end{align}
	\begin{align}
		\label{eq:small-coeff-36}
		2\frac{a'_{i,i}}{\tilde{a}_{i,j}} + 2\frac{a'_{j,j}}{\tilde{a}_{i,j}} + 4\frac{a'_{i,j}}{\tilde{a}_{i,j}} - \frac{2c_{1,i}}{\sqrt{2\tilde{a}_{i,j}}} - \frac{2c_{1,j}}{\sqrt{2\tilde{a}_{i,j}}} + c_2-1+c_4 - \frac{4\beta_i}{\sqrt{2\tilde{a}_{i,j}}} - \frac{4\beta_j}{\sqrt{2\tilde{a}_{i,j}}} > 0
	\end{align}
	Combining (\ref{eq:small-coeff-33}) and (\ref{eq:small-coeff-36}) we get
	\begin{align}
		\label{eq:small-coeff-37}
		2\frac{a'_{i,i}}{\tilde{a}_{i,j}} + 2\frac{a'_{j,j}}{\tilde{a}_{i,j}} + 4\frac{a'_{i,j}}{\tilde{a}_{i,j}} + c_2 - 1 + c_4 > 0
	\end{align}
	From this we get that
	\begin{align}
		\label{eq:small-coeff-38}
		\frac{a'_{i,j}}{\tilde{a}_{i,j}} > \frac 1 4 - \delta - \frac \epsilon 4
	\end{align}
	for large enough $\tau$.
	Similarly, combining (\ref{eq:small-coeff-34}) and (\ref{eq:small-coeff-35}) we get
	\begin{align}
		\label{eq:small-coeff-39}
		2\frac{a'_{i,i}}{\tilde{a}_{i,j}} + 2\frac{a'_{j,j}}{\tilde{a}_{i,j}} - 4\frac{a'_{i,j}}{\tilde{a}_{i,j}} + c_2 - 1 + c_4 < 0
	\end{align}
	From this we get that
	\begin{align}
		\label{eq:small-coeff-40}
		\frac{a'_{i,j}}{\tilde{a}_{i,j}} > -\frac 1 4 - \delta - \frac \epsilon 4
	\end{align}
	for large enough $\tau$.

\end{proof}


\subsection{Proof of Weak Robust Learning}\label{app:weaker-hardness}

In this section we prove Theorem~\ref{THM:lowerbound-approx}, which in turns uses the non-distributional hardness in Theorem~\ref{THM:lowerbound-strong-nondistributional}. 
But to begin with we first prove an alternate NP hardness result. Although weaker than the hardness result of the previous section, this will help us prove the more robust bound. More formally, we will prove that
\begin{theorem}\label{THM:lowerbound}[Hardness]
For every $\delta>0$, assuming $NP \ne RP$ there is no polynomial time algorithm that given a set of $N=O(n^{2})$ labeled points $\set{(x^{(1)},y^{(1)}), \dots, (x^{(N)}, y^{(N)})}$ with $(x^{(j)}, y^{(j)}) \in \R^{n+1} \times \set{-1,1}$ for all $j \in [N]$ can determine whether there exists a degree-$2$ PTF that has $\delta$-robust empirical error of $0$ on these $N$ points.

\end{theorem}
The above theorem immediately implies the following result about hardness of optimal robust learning of degree-$2$ PTFs.
\begin{corollary}\label{cor:lowerbound}[Distributional Hardness]
For every $\delta>0$, there exists an $\epsilon>0$ such that assuming $NP \ne RP$ there is no algorithm 
that given a set of $N=poly(n, \frac{1}{\epsilon})$ samples from a distribution $D$ over $\mathbb{R}^n \times \{-1, +1\}$, runs in time $\poly(N)$ and distinguishes between the following two cases:
\begin{itemize}
\item {\sc Yes:} There exists a degree-$2$ PTF that has $\delta$-robust error of $0$ w.r.t. $D$.
\item {\sc No:} There exists no degree-$2$ PTF that has $\delta$-robust error at most $\epsilon$ w.r.t. $D$.
\end{itemize}

\end{corollary}

We again reduce from the QP problem (Problem \QP) which is known to be NP hard. The reduction is sketehced below.
\begin{figure}[H]
	\begin{center}
		\fbox{\parbox{0.98\textwidth}{
				\begin{enumerate}
					\item Let $p(x):=x^{T} A x$ be the polynomial given by Problem \QP, and let $\beta, \delta$ be the given parameters. Set $\alpha:=\delta^2 \beta + \delta, \rho:=c_3 \delta n^{3/2} m$, for some sufficiently large constant $c_3 \ge 1$.
					\item Using $A$ we generate $m$ points $(x^{(j)},z^{(j)}) \in \R^{n+1}$ as follows. Sample point $x^{(j)}$ from $\N(0,\rho^2)^n$, then set $z^{(j)} = p(x^{(j)})=(x^{(j)})^T A x^{(j)}$ for each $j \in [m]$.

					\item Define $s^{(j)} = \sgn(\nabla p(x^{(j)}) )$ where the $\sgn(x) \in \set{-1,1}^n$ refers to a vector with entry-wise signs, and $\nabla p$ stands for the gradient of $p$ at $x^{(j)}$. From each $(x^{(j)},z^{(j)})$ generate $(u^{(j)},z^{(j)}_u) = (x^{(j)}-\delta s^{(j)},z^{(j)}+\delta)$ with label $y^{(j)}_u = \sgn(z^{(j)}_u - p(u^{(j)}) )$ and $(v^{(j)},z^{(j)}_v) = (x^{(j)}+\delta s^{(j)},z^{(j)}-\delta)$ with label $y^{(j)}_v = \sgn(z^{(j)}_v - p(v^{(j)}))$.
				\item Generate $\alpha$ (depends on $\delta$ and $\beta$ from problem \QP) and input the $2m+1$ points in $\R^{n+1} \times \set{\pm 1}$ given by $((u^{(j)},z^{(j)}_u), y^{(j)}_u)$, $((v^{(j)},z^{(j)}_v), y^{(j)}_u)$ for each $j \in [m]$ and $(0,\alpha, +1)$ to the algorithm. 

				\end{enumerate}
		}}
	\end{center}
	\caption{Reduction from the QP problem.}\label{ALG:reduce}
\end{figure}


\begin{figure}
\centering

    \includegraphics[width=0.7\textwidth]{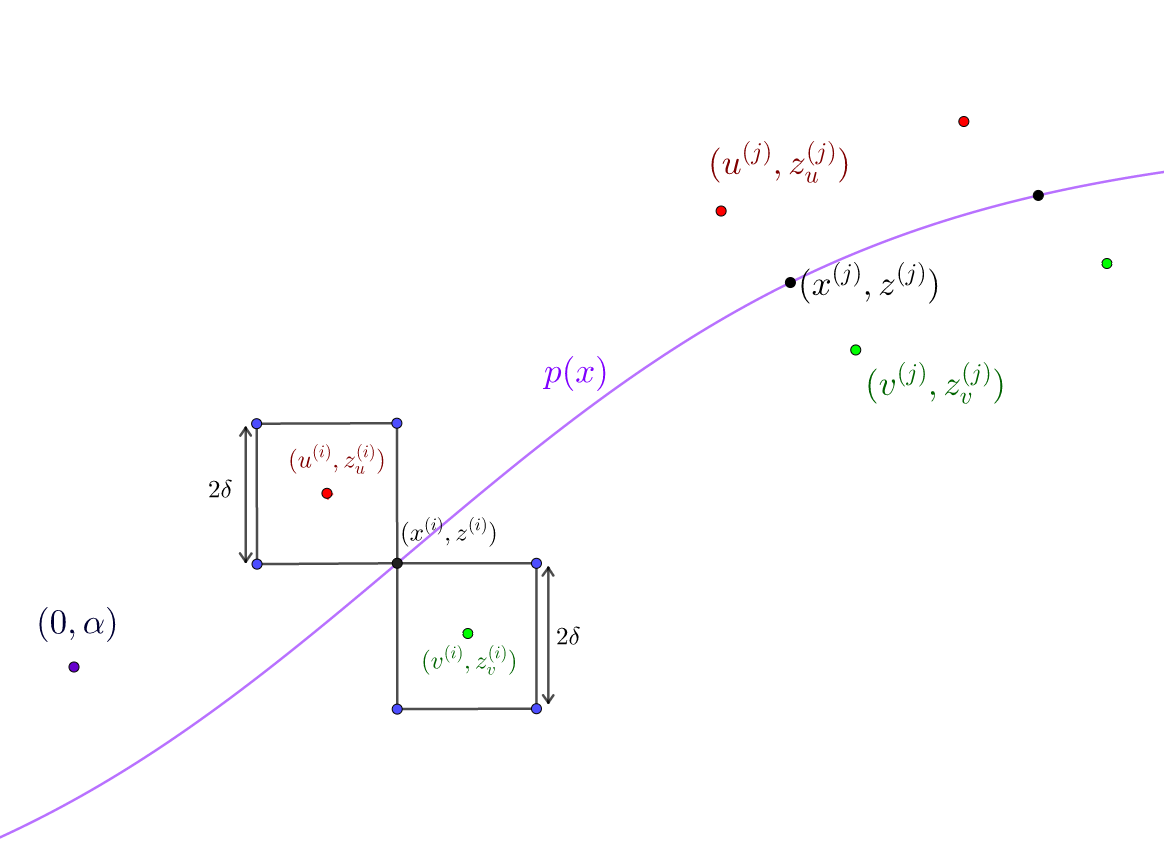}

\caption{{\em The figure shows the construction of a hard instance for the robust learning problem. First, points $(x^{(j)}, z^{(j)})$ are sampled randomly and staisfying $z^{(j)} = p(x^{(j)})$. Each such point is then perturbed along the direction of the sign vector of the gradient at $(x^{j}, z^{(j)})$ to get two data points of the training set, one labeled as $+1$, and the other labeled as $-1$.} \label{fig:1}}

\vspace{-10pt}
\end{figure}

To argue the soundness and the completeness of our reduction, we will first analyze the robustness of degree-$2$ PTFs on the $2m$ added labeled examples $((u^{(\ell)}, z^{(\ell)}_u), y^{(\ell)}_u)$ and $((v^{(\ell)}, z^{(\ell)}_v), y^{(\ell)}_v)$. We will show that the ``intended'' PTF $\sgn(z-p(x))$ is the {\em unique} degree-$2$ PTF (up to scaling) that is robust at all these $2m$ points. Note that a degree-$2$ PTF $\sgn(q(x,z))$ on the $n+1$ variables $(x,z)$ may {\em not} necessarily be of the form $\sgn(z-g(x))$ for some degree-$2$ polynomial $g(x)$. We need to rule out the existence of any other degree-$2$ PTF of the form $\sgn(q(x,z))$ that is $\delta$-robust at these points. Once we have established this, we will then show that the ``intended'' PTF $\sgn(z-p(x))$ is $\delta$-robust at $((0,\alpha),+1)$ in the \YES case, and not $\delta$-robust at $((0,\alpha),+1)$ in the \NO case.  

We proceed by first proving that the intended PTF $\sgn(z-p(x))$ is robust at the $2m$ added examples. 
Recall that the points $x^{(j)} \in \R^n$ are chosen according to a Gaussian distribution with variance $\rho^2$ in every direction. The following lemma shows a property that holds w.h.p. for the points $\set{x^{(\ell)}: \ell \in [m]}$ that will be key in proving the robustness of $\sgn(z-p(x))$ at the $2m$ added points in Lemma~\ref{LEM:robust}. 

\begin{lemma}\label{LEM:randomproperty}
There exists some universal constant $C>0$ such that for any $\eta>0$, assuming $\rho \ge C \delta n^{3/2} m/\eta$ we have with probability at least $1-\eta$ that 
\begin{equation}\label{eq:randomproperty}
\forall \ell \in [m],~ \forall i \in [n], ~ \frac{|\iprod{A_i,x^{(\ell)}}|}{\norm{A_i}_1} > \delta,  
\end{equation}  
where $A_i$ denotes the $i$th row of $A$. 
\end{lemma}
\begin{proof}
The proof follows from the following standard anti-concentration fact about Gaussians.

\begin{fact}\label{FACT:Gaussian}
	Let $x^*$ be sampled from $\N(0,\rho^2)^n$. Let $a \in \R^n$. There exists a universal constant $C>0$ such that for any $\eta'>0$,
		$$\Pr\Big[ |\iprod{a,x^*}| \leq C\norm{a}_2 \rho \eta \Big] \le \eta'.$$
\end{fact}

Set $\eta':=\eta/(mn)$. Fix $\ell \in [m], i \in [n]$. Using Fact~\ref{FACT:Gaussian} we have with probability at least $1-\eta'$
\begin{align*}
|\iprod{A_i, x^{(j)}}| \ge  \norm{A_i}_2 \rho \eta' \ge \frac{\norm{A_i}_1}{\sqrt{n}} \cdot \rho \cdot \frac{\eta}{mn} \ge \delta, 
\end{align*}
from our assumption on $\rho$. The lemma follows from a union bound over all $\ell \in [m], i \in [n]$.  
\end{proof}

We now prove the $\delta$-robustness of the ``intended'' degree-$2$ PTF $\sgn(z-p(x))$ at the $2m$ added points w.h.p. 
\begin{lemma}\label{LEM:robust}
	There exists constant $C>0$ such that for any $\eta>0$, assuming $\rho \ge C \delta n^{3/2} m/\eta$, then with probability at least $1-\eta$, the degree-$2$ PTF $\sgn(z-p(x))= \sgn(z-x^T A x)$ is $\delta$-robust at all the $2m$ points $\set{((u^{(\ell)},z^{(\ell)}_u), y^{(\ell)}_u), ((v^{(\ell)},z^{(\ell)}_v), y^{(\ell)}_v):\ell \in [m]}$. 
\end{lemma}

\begin{figure}
\centering

    \includegraphics[width=0.7\textwidth]{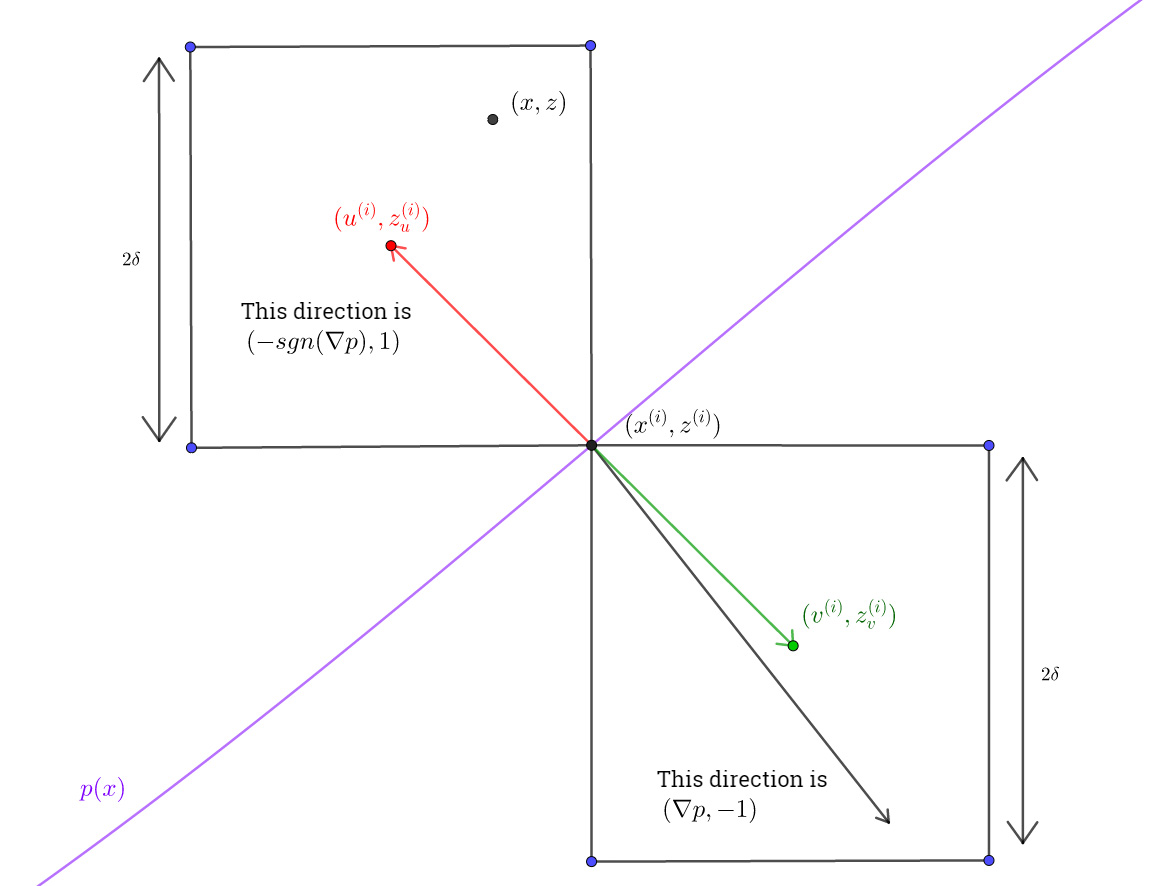}

\caption{{\em The figure shows the radius of robustness around the point $(x^{(i)}, z^{(i)})$}. Any degree-$2$ PTF that is $\delta$-robust at all the data points, must take a value of $+1$ in the upper ball around each $(x^{(i)}, z^{(i)})$ of $\ell_\infty$ radius of $2\delta$ and must take a value of $-1$ in the lower ball around each $(x^{(i)}, z^{(i)})$ of $\ell_\infty$ radius of $2\delta$. We use this fact to establish that such a PTF must pass through the points $(x^{(i)}, z^{(i)})$.} \label{fig:2}

\vspace{-10pt}
\end{figure}

\begin{proof}
	Consider a fixed $\ell \in [m]$. For convenience let $x^*,z^*,u,v, z_u, z_v$ denote $x^{(\ell)}, z^{(\ell)}, u^{(\ell)},$ $v^{(\ell)}, z^{(\ell)}_u, z^{(\ell)}_v$ respectively, and let $s= \sgn(\nabla p(x^{(\ell)})) \in \set{-1,1}^n$.  
Hence $z^* = x^{*T}Ax^*$,  $(u,z_u) = (x^*-\delta s, z^*+\delta)$ and $(v,z_v) = (x^*+\delta s, z^*-\delta )$. We want to prove that the points $(u,z_u)$ and $(v,z_v)$ are $\delta$ robust i.e.,  these points are $\delta$ away in $\ell_\infty$ distance from the decision boundary of $\sgn(z-p(x))$. 
We now prove the following claim:

\noindent {\em{\bf Claim.} Any point $(x,z) \in B^{n+1}_\infty(u,z_u)$ is on the `positive' side i.e., $z-x^T A x >0$.} 

Note that $(u,z_u)$ itself lies inside the ball, and hence the claim will show that $\sgn(z-x^{T} A x)$ is $\delta$-robust at $(u,z_u)$. An analogous proof also holds that $\delta$-robustness at $(v,z_v)$. 

\noindent {\em Proof of Claim.} Let's now define $\tilde{x} = x-x^*$, $\tilde{z} = z-z^*$. 
	A simple observation is that $(x,z)$ lies on the opposite orthant with respect to $(x^*,z^*)$ as $s$ , 
	and we have (as shown in Figure~\ref{fig:2}) 
	$$\forall j \in [d], -2\delta\leq s(j)\tilde{x}(j) \leq 0, ~\text{ and } \tilde{z} \ge 0.$$ 
Using $z^*=p(x^*)$ and $\tilde{z} \ge 0$, for all $(x,z) \in B^{n+1}((u,z_u),\delta)$ we have  
\begin{align*}	
	z-p(x)&=z^* + \tilde{z} - p(\tilde{x} + x^*) = \tilde{z} + p(x^*) - p(\tilde{x}+x^*) = \tilde{z} - \inner{\nabla p,\tilde{x}} - \frac{1}{2}\tilde{x}^T\nabla^2p\tilde{x}\\
	&\ge -\sum_{i=1}^n \tilde{x}(i) \Big(\sum_{j=1}^n a_{ij} x^*(j) \Big) - \frac{1}{2}\sum_{i=1}^n \tilde{x}(i) \Big(\sum_{j=1}^n a_{ij} \tilde{x}(j) \Big)\\
	& = \sum_{i=1}^n (-\tilde{x}(i) s(i)) \Bigabs{\sum_{j=1}^n a_{ij} x^*(j) } - \frac{1}{2}\sum_{i=1}^n \tilde{x}(i) \sum_{j=1}^n a_{ij} \tilde{x}(j) \\
	&\ge \sum_{i=1}^n |\tilde{x}(i)| \Big( \Bigabs{\sum_{j=1}^n a_{ij} x^*(j) } - \delta \sum_{j=1}^n |a_{ij}| \Big),
\end{align*}	 
using the fact that $\tilde{x}(i)s(i) \in [-2\delta, 0]$ for each $i \in [n]$. Applying Lemma~\ref{LEM:randomproperty} we see that with probability at least $(1-\eta)$, \eqref{eq:randomproperty} holds, and hence $|\iprod{x^*,A_i}| > \delta \norm{A_i}_1$ for each $i \in [n]$ as required. This establishes the claim, and proves the lemma.

\end{proof}



We now prove that the ``intended'' PTF $\sgn(z-p(x))$ is the only degree-$2$ PTF that is robust at the added $2m$ examples.
\begin{lemma}\label{LEM:unique}
	Consider any degree-$2$ PTF $\sgn(q(x,z))$ that is $\delta$-robust at the $2m$ labeled points  $\set{((u^{(\ell)},z^{(\ell)}_u), +1): \ell \in [m]}$ and $\set{((v^{(\ell)},z^{(\ell)}_v),-1): \ell \in [m]}$ and is consistent with their labels. Then $q(x,z)= C ( z-p(x))$ for some $C \ne 0$. 
\end{lemma}

The proof of Lemma~\ref{LEM:unique} follows immediately from the following two lemmas (Lemma~\ref{LEM:passingthrough} and Lemma~\ref{LEM:uniqueness}).
	\begin{lemma}\label{LEM:passingthrough}
Consider any degree-$2$ PTF on $n+1$ variables $\sgn(q(x,z))$ that satisfies the conditions of Lemma~\ref{LEM:unique}. Then $q(x^{(\ell)}, z^{(\ell)})=0$ for each $\ell \in [m]$. 
\end{lemma}
	
\begin{proof}
		Since $\sgn(q(u^{(\ell)},z^{(\ell)}_u))\neq sgn(q(v^{(\ell)},z^{(\ell)}_v))$, by the Intermediate Value Theorem, 
		$$\exists \gamma \in [0,1] \text{ s.t. } (\widehat{x},\widehat{z})= \gamma(u^{(\ell)},z^{(\ell)}_u)+ (1-\gamma)(v^{(\ell)},z^{(\ell)}_v) \text{ and } q(\widehat{x},\widehat{z})= 0.$$ 
		
		Also, since $q$ is $\delta$-robust at $(u^{(\ell)},z^{(\ell)}_u)$ and $(v^{(\ell)},z^{(\ell)}_v)$, we must have that $(\widehat{x},\widehat{y})$ is atleast $\delta$ far away in $\ell_\infty$ distance from both $(u^{(\ell)},z^{(\ell)}_u)$ and $(v^{(\ell)},z^{(\ell)}_v)$. Further by design two points are separated by exactly $2\delta$ in each co-ordinate (see Figure~\ref{fig:2} for an illustration)! Hence it is easy to see that $\gamma=1/2$ i.e., $(\widehat{x}, \widehat{z})=(x^{(\ell)},z^{(\ell)})$ as required.

	\end{proof}
	

We now show that $q(x,z)=z-p(x)$ is the only polynomial over $(n+1)$ variables that evaluates to $0$ on all points $\set{(x^{(\ell)}, z^{(\ell)}): \ell \in [m]}$. Together with Lemma~\ref{LEM:passingthrough} this establishes the proof of Lemma~\ref{LEM:unique}.

	\begin{lemma}\label{LEM:uniqueness}
Let $m > (n+1)^2$ and let $q:\R^{n+1} \to \R$ be any degree-$2$ polynomial with $q(x^{(\ell)}, z^{(\ell)})=0$ for all $\ell \in [m]$, where $z^{(\ell)}=(x^{(\ell)})^T A^* x^{(\ell)}$ and $x^{(\ell)} \sim N(0,\rho^2)^n$ with $\rho>0$.  Then with probability $1$, $q(x,z)= C (z- x^T A^* x)$ for $C \ne 0$.  	
	\end{lemma}
	
	\begin{proof}
We can represent a general degree-$2$ polynomial $q:\R^{n+1} \to \R$ given by 
	$$ q(x,z) = x^T A x + b_1^Tx + c_1 + z b_2^Tx + c_2z^2 + c_3z, \text{ where } x \in \R^n, z \in \R. $$
	This polynomial is parameterized by a vector $w=(A, b_1, c_1, b_2, c_2, c_3) \in \R^{r}$ where $r={n+1 \choose 2}+2n+3$ (since $A$ is symmetric). 
		Now given a point $(x^{(\ell)},z^{(\ell)})$, the equation $q(x^{(\ell)},z^{(\ell)}) = 0$ is a linear equation over the coefficients $w$ of $q$. Hence, the set of conditions $q(x^{(\ell)}, z^{(\ell)})=0$ can be expressed as a systems of linear equations $Mw =0$ over the (unknown) co-efficients $w$. Hence any valid polynomial $q$ corresponds to a solution of the linear system $Mw=0$ and vice-versa. We now describe the matrix $M \in \R^{m \times r}$. Define 
\begin{align*}
f(x,z)  &:=(1)\oplus(x_1,\dots, x_n)\oplus (x_i x_j: i \le j \in [n]) \oplus (x_1 z, \dots x_n z) \oplus (z^2), \oplus(z) \in \R^{r},\\
\text{ and } M_\ell &:= f(x^{(\ell)},z^{(\ell)})~ \forall \ell \in [m],  
\end{align*}
where $u \oplus v$ refers to the concatenation of vectors $u$ and $v$, and $M_\ell$ represents the row $\ell$ of $M$. In other words $f(x,z)=(1, x_1,\dots, x_n, x_1^2, \dots, x_jx_k, \dots ,x_n^2, x_1 z,\dots, x_j z \dots ,x_n z,  z^2, z )$, where $x_j$ is the $j$th component of $x$ and $z=x^TA^*x$. 
Observe that the ``intended'' polynomial $q^*(x,z)=z-x^T A^* x$ is a valid solution to this system of equations. Hence, it will suffice to prove that $M$ has rank exactly $r-1$ i.e., $M$ has full column rank minus one. 
First observe that as polynomials over the formal variables $x,z$, {\em all but one} of the columns of $f$ are linearly independent --  in fact the only linear dependency in $f(x,z)$ corresponds to the column $z$ that can be expressed as a linear combination of degree-$2$ monomials $\set{x_ix_j: i \le j}$ since $z:=x^T A^* x$ is a homogenous degree-$2$ polynomial. Further the columns $\set{x_j z: j \in [n]}$ have degree $3$ and $z^2$ has degree $4$. Hence excluding the column corresponding to $z$, it is easy to see that the rest of the columns are linearly independent (either they correspond to different monomials, or the degrees are different).  
Now define $g(x,z), M'$ analogously to $f(x,z)$ and $M$ respectively, without the last column that corresponds to $z$ i.e., 
\begin{align*}
g(x,z)  &:=(1)\oplus(x_1,\dots, x_n)\oplus (x_i x_j: i \le j \in [n]) \oplus (x_1 z, \dots x_n z) \oplus (z^2)  \in \R^{r-1},\\
\text{ and } M'_\ell &:= g(x^{(\ell)},z^{(\ell)}) ~\forall \ell \in [m].  
		\end{align*}
From our earlier discussion, the columns of $g(x,z)$ when seen as polynomials over the formal variables $x,z$ are linearly independent. 
Hence, it suffices to prove the following claim:

\noindent {\em {\bf Claim:} $M'$ has full column rank i.e., rank of $M'$ is $r$.  }


To see why the claim holds consider the first $\ell$ rows of the matrix $M'$ and look at their span $S(R_\ell)$. If $\ell \le r-1$ then the space orthogonal to $S(R_\ell)$ i.e., $S(R_\ell)^\perp$ is non-empty. Consider any direction $v$ in $S(R_\ell)^\perp$. 
		$$\inner{v,M'_{\ell+1}}=\widehat{q}(x^{(\ell+1)},z^{(\ell+1)}), \text{ where } \widehat{q}(x,z):=\inner{v,g(x,z)}$$ 
is a non-zero polynomial of degree $2$ in $x,z$ (it is not identically zero because the columns of $g(x,z)$ are linearly independent as polynomials over $x,z$). Hence using a standard result about multivariate polynomials evaluated at randomly chose points~(See Fact~\ref{FACT:generic}), we get that $\widehat{q}(x^{(\ell+1)},z^{(\ell+1)}) \neq 0$ and so $\inner{v,M'_{\ell+1}}\neq 0$ with probability $1$. Taking a union bound over all the $\ell \in \set{1, \dots, r}$ completes the proof.  

	\end{proof}
	\begin{fact}\label{FACT:generic}
		A non-zero multivariate polynomial ${p}:\R^n \to \R$ evaluated at a point $x \sim N(0,\rho^2)^n$ with $\rho>0$ evaluates to zero with zero probability.
	\end{fact}	

We remark that the statement of Lemma~\ref{LEM:uniqueness} can also be made robust to inverse polynomial error by using polynomial anti-concentration bounds (e.g., Carbery-Wright inequality) instead of Fact~\ref{FACT:generic}; however this is not required for proving NP-hardness. 
We now complete the proof of Theorem~\ref{THM:lowerbound}.

\begin{proof}[Proof of Theorem~\ref{THM:lowerbound}]
We start with the NP-hardness of \QP, and for the reduction in Figure~\ref{ALG:reduce}, we will show that in the \YES case, we will show that there is a $\delta$-robust degree-$2$ PTF  (completeness), and in the \NO case we will show that there is no $\delta$ robust degree-$2$ PTF (soundness). As a reminder, the NP-hard problem {\QP} is the following: given a symmetric matrix $A \in \R^{n \times n}$ with zeros on diagonals, and $\beta>0$ distinguish whether 

\textbf{\NO Case} : there exists an assignment $y^*$ with $\norm{y^*}_\infty  \le 1$ such that $q(y^*)= (y^{*})^T A y^* >  \beta$, 

\textbf{\YES Case} : $\max_{\norm{y}_\infty \le 1 } y^T A y < \beta$. 


\noindent{\em Completeness }(\textbf{\YES Case}): Consider the degree-$2$ PTF given by $\sgn(z-p(x))=\sgn(z-x^T A x)$. From Lemma~\ref{LEM:robust}, we have that it is $\delta$ robust at the $2m$ points $\set{((u^{(\ell)}, z^{(\ell)}_u), y^{(\ell)}_u): \ell \in [m]}$ and $\set{((v^{(\ell)}, z^{(\ell)}_v), y^{(\ell)}_v): \ell \in [m]}$ with probability at least $1-\eta$ (for $\eta$ being any sufficiently small constant).  Further, from multilinearity of $p$ we have that,
\begin{align*}
\max_{\norm{y}_\infty \le \delta} y^{T} A y &= \delta^2 \max_{\norm{y}_\infty \le 1} y^{T} A y < \delta^2 \beta = \alpha -\delta. \\
\text{Hence } &(\alpha - \delta) - \max_{\norm{y}_\infty \le \delta } y^T A y >0,  
\end{align*}
which establishes robustness at $((0,\alpha),+1)$ for $\sgn(z-x^T A x)$.  Hence $\sgn(z-p(x))$ is $\delta$-robust at the $2m+1$ points with probability at least $1-\eta$ (for $\eta$ being any sufficiently small constant). 

\noindent{\em Soundness }(\textbf{\NO Case}): 
From Lemma~\ref{LEM:unique}, we see that the degree-$2$ PTF given by $\sgn(z-p(x))=\sgn(z-x^T A x)$ is the only degree-$2$ PTF that can potentially be robust at all the $2m+1$ points with probability $1$. Again analyzing robustness at the example $((0,\alpha),+1)$, we see that from multilinearity of $p$,
\begin{align*}
\max_{\norm{y}_\infty \le \delta} y^{T} A y &= \delta^2 \max_{\norm{y}_\infty \le 1} y^{T} A y > \delta^2 \beta = \alpha -\delta. \\
\text{Hence } &(\alpha - \delta) - \max_{\norm{y}_\infty \le \delta } y^T A y < 0,  
\end{align*}
which shows that the degree-$2$ PTF $\sgn(z-p(x))$ is {\em not} robust at $(0,\alpha)$. Hence there is no $\delta$-robust degree-$2$ PTF at the $2m+1$ given points, with probability $1$. This completes the analysis of the reduction, and establishes the theorem.

\end{proof}


\paragraph{Stronger Hardness.}


We now prove the robust lower bound stated below.
\begin{theorem}\label{THM:lowerbound-strong-nondistributional}[Stronger  Hardness]
For every $\delta>0$ and $\epsilon \in (0,\tfrac{2}{7})$, assuming $NP \ne RP$ there is no polynomial time algorithm that given a set of $N=\poly(n,1/\epsilon)$ labeled points $\set{(x^{(1)},y^{(1)}), \dots, (x^{(N)}, y^{(N)})}$ in $\R^{n+1} \times \set{-1,1}$ such that there is a degree-$2$ PTF with $\delta$-robust empirical error of $0$, can output a degree-$2$ PTF that has $\delta$-robust empirical error of at most $\epsilon$ on these $N$ points.
\end{theorem}


\begin{proof}  
The proof of this theorem closely follows the proof of Theorem~\ref{THM:lowerbound} (the $\varepsilon=0$ case), so we only point out the differences here. The reduction uses the 
same gadget (Figure~\ref{ALG:reduce}) used in Theorem~\ref{THM:lowerbound}. 
The main challenge is the soundness analysis (NO case), where we need to rule out the existence of degree-$2$ PTFs which are $\delta$-robust and consistent at all but an $\varepsilon$ fraction of the points. To handle this, we introduce ``redundancy'' by including more points (of both kinds) to ensure that even when an arbitrary $\varepsilon$ fraction of these points are ignored (the PTF makes errors on them), we can still use the arguments in the soundness analysis of Theorem~\ref{THM:lowerbound}. 


Recall that our reduction (see Figure~\ref{ALG:reduce}) generated two sets of points. We have one point of the form $(0,\alpha)$ (let us denote this type as {\em Type A}) and $m$ pairs of points 
$\set{(u^{(\ell)}, z_u^{(\ell)} ),(v^{(\ell)}, z_v^{(\ell)}): \ell \in [m]}$ 
which are obtained by modifying $(x^{(\ell)}, z^{(\ell)}=p( x^{(\ell)}) )$ with $x^{(\ell)}$ generated randomly (let us denote these $2m$ points as of {\em Type B}).   

Set $N_1:=n^3, N_2:= 2n^3$. 
In our modified instance, we will have $N_1$ points of Type A i.e., $N_1$ identical points $(0,\alpha)$ (note that we can also perturb these points slightly so that they are all distinct, if required). Further, we will have $N_2$ points of Type B i.e., we will generate $N_2/2$ pairs of points $\set{(u^{(\ell)}, z_u^{(\ell)}): \ell \in [N_2/2]}$ which are generated as described in Figure~\ref{ALG:reduce} after drawing $x^{(\ell)} \sim N(0,\rho^2)^n$ for $\ell \in [N_2/2]$ (here a larger $\rho=O(\delta n^{3/2} N_2)$ will suffice).  Hence, we have in total $N=N_1 + N_2 = 3n^3$ points.

The {\em completeness} analysis (YES case) is identical to that of Theorem~\ref{THM:lowerbound}, as $\sgn(z-p(x))$ will be $\delta$-robust at all of the $N$ points (from Lemma~\ref{LEM:robust} and our choice of $\alpha$). 

We now focus on the {\em soundness} analysis (NO case).  
From $\varepsilon < \tfrac 1 3$ and our choice of $N_1$ and $N_2$, 
\begin{align}
N_1 & > \varepsilon (N_1+N_2) \label{eq:stronger:sound1}\\
(1-\varepsilon)(N_1+N_2) &> N_1 + \frac{N_2}{2} + (n+1)^2 \label{eq:stronger:sound2}
\end{align}



From \eqref{eq:stronger:sound2} and a pigeonhole argument, any subset of size $(1-\varepsilon)(N_1+N_2)$ is guaranteed to have $(n+1)^2$ pairs of points of the form  $(u^{(\ell)},z^{(\ell)}_u)$ and $(v^{(\ell)},z^{(\ell)}_v)$. This is because the LHS of \eqref{eq:stronger:sound2} represents a lower bound on the number of points the candidate degree-$2$ PTF is robust on. The RHS of \eqref{eq:stronger:sound2} represents the number of points needed to ensure that atleast $(n+1)^2$ pairs of points from Type B are picked.
Hence using Lemma \ref{LEM:unique} along with a union bound over all the ${N_2 \choose (n+1)^2}$ choices of the pairs (note that the failure probability in Lemma~\ref{LEM:uniqueness} is $0$), the ``intended'' PTF $\sgn(z-p(x))$ is the only surviving degree-$2$ PTF. 

Again from \eqref{eq:stronger:sound1} and the pigeonhole principle, any $(1-\varepsilon)$ fraction of the points will contain {\em atleast one} point of the Type A i.e., $(0,\alpha)$. Hence in the NO case, the ``intended'' PTF $\sgn(z-p(x))$ is not $\delta$-robust. This completes the soundness analysis and establishes the theorem.

\end{proof}



\end{document}